\documentclass{article}

\usepackage{xcolor}
\usepackage{color}
\usepackage{amsmath}
\usepackage{amssymb}
\usepackage{mathtools}
\usepackage{amsthm}
\usepackage{multirow}
\usepackage{algorithm}
\usepackage{algorithmic}

\definecolor{DarkGreen}{rgb}{0.1,0.5,0.1}
\definecolor{DarkRed}{rgb}{0.5,0.1,0.1}
\definecolor{DarkBlue}{rgb}{0.1,0.1,0.5}
\definecolor{Gray}{rgb}{0.2,0.2,0.2}

\definecolor{negvals}{rgb}{0.1,0.5,0.1}
\definecolor{neutralvals}{RGB}{255, 140, 0}
\definecolor{neutralvals}{rgb}{0.43, 0.5, 0.5}
\definecolor{posvals}{rgb}{0.8, 0.2, 0.0}

\definecolor{eval}{rgb}{0.1,0.45,0.1}
\definecolor{optim}{rgb}{0.8,0.3,0.0}

\DeclareMathOperator*{\E}{\mathbb{E}}
\DeclareMathOperator*{\prob}{\mathbb{P}}

\newcommand{\Ind}{\mathbb I}

\newcommand{\defeq}{\stackrel{\small \mathrm{def}}{=}}

\DeclareMathOperator*{\argmin}{argmin}
\DeclareMathOperator*{\argmax}{argmax}

\newcommand{\cD}{{\cal D}}
\newcommand{\cE}{{\cal E}}

\newcommand{\cN}{{\cal N}}

\newcommand{\cP}{{\cal P}}

\newcommand{\cS}{{\cal S}}

\newcommand{\cV}{{\cal V}}
\newcommand{\cW}{{\cal W}}
\newcommand{\cX}{{\cal X}}

\newcommand{\vn}{\mathbf n}

\newcommand{\norm}[1]{\left\lVert#1\right\rVert}
\usepackage{tikz}
\usetikzlibrary{positioning}
\usetikzlibrary{arrows}
\usetikzlibrary{intersections}
\usetikzlibrary{matrix,backgrounds}
\newcommand{\R}{\mathbb R}
\newcommand{\indep}{\perp \!\!\! \perp}

\newcommand{\WDM}{\operatorname{WDM}}

\newcommand{\GRDR}{\operatorname{GRDR}}

\usepackage{wrapfig}
\newcommand{\avgm}{ \sum_{i=1}^m  }
\newcommand{\summ}{ \sum_{i=1}^m  }
\newcommand{\sumt}{ \sum_{t \in \{0,1\} }}
\newcommand{\avgn}{ n^{-1} \sum_{i=1}^n  }

\newcommand{\gxt}[1]{\mu(W_i;#1)}

\newcommand{\gfn}{\mu}

\newcommand{\tht}{\perturbparam}
\newcommand{\hth}{\hat{\perturbparam}}
\newcommand{\thtph}{\tht^*_{\pathg+h}}
\newcommand{\thtp}{\tht^*_\pathg}
\newcommand{\z}{x}
\newcommand{\pathg}{\epsilon}

\newcommand{\ts}{\textspace}
\def\textspace{\openup 2\jot\relax}
\usepackage{enumitem}

\newcommand{\pith}{\varphi}

\newcommand{\oraclemuv}{\tilde{v}} 
\newcommand{\estmu}{\hat{\mu}} 
\newcommand{\estmuv}{\hat{v}} 

\newcommand{\perturbparam}{\xi} 
\newcommand{\parammu}{\theta}
\newcommand{\parame}{\gamma}
\newcommand{\genest}{\diamond}

\usepackage{subfig}
\usepackage{graphicx}
\usepackage{pgfplots}
\pgfplotsset{compat=1.12}

\usepackage[hidelinks]{hyperref}
\usepackage{nicefrac} 
\usepackage[capitalize,noabbrev]{cleveref}
\crefname{figure}{fig.}{figs.}
\Crefname{figure}{Fig.}{Figs.}

\theoremstyle{plain}
\newtheorem{theorem}{Theorem}[section]
\newtheorem{proposition}[theorem]{Proposition}

\theoremstyle{definition}

\newtheorem{assumption}[theorem]{Assumption}
\newtheorem{example}[theorem]{Example}

\theoremstyle{remark}
\newtheorem{remark}[theorem]{Remark}

\usepackage[final]{neurips_2022}

\usepackage[utf8]{inputenc} 
\usepackage[T1]{fontenc}    
\usepackage{url}            
\usepackage{booktabs}       
\usepackage{amsfonts}       
\usepackage{nicefrac}       
\usepackage{microtype}      

\usepackage{microtype}
\usepackage{graphicx}
\usepackage{booktabs} 
\usepackage[utf8]{inputenc}
\usepackage[T1]{fontenc}
\usepackage{graphicx}
\usepackage{url}            
\usepackage{amsfonts}       
\usepackage{nicefrac}       
\usepackage{lipsum}
\usepackage{float}
\usepackage{array}
\usepackage{xspace}
\usepackage{multirow}
\usepackage[export]{adjustbox}
\usepackage{stfloats}
\usepackage{xcolor}         
\title{Off-Policy Evaluation with Policy-Dependent Optimization Response}

\author{%
 Wenshuo Guo$^*$\\
  Department of EECS\\
  University of California, Berkeley\\
  \texttt{wguo@cs.berkeley.edu} \\
   \And
   Michael I. Jordan${^*}$ \\
   Department of EECS and Department of Statistics \\
   University of California, Berkeley\\
   \texttt{jordan@cs.berkeley.edu}\\
   \And
   Angela Zhou\thanks{Authors listed in alphabetical order.}\\
   Department of Data Sciences and Operations\\
   University of Southern California\\
   \texttt{zhoua@usc.edu}
}

\begin{document}

\maketitle
\begin{abstract}
The intersection of causal inference and machine learning for decision-making is rapidly expanding, but the default decision criterion remains an \textit{average} of individual causal outcomes across a population. In practice, various operational restrictions ensure that a decision-maker's utility is not realized as an \textit{average} but rather as an \textit{output} of a downstream decision-making problem (such as matching, assignment, network flow, minimizing predictive risk). In this work, we develop a new framework for off-policy evaluation with \textit{policy-dependent} linear optimization responses: causal outcomes introduce stochasticity in objective function coefficients. Under this framework, a decision-maker's utility depends on the policy-dependent optimization, which introduces a fundamental challenge of \textit{optimization} bias even for the case of policy evaluation. We construct unbiased estimators for the policy-dependent estimand by a perturbation method, and discuss asymptotic variance properties for a set of adjusted plug-in estimators. Lastly, attaining unbiased policy evaluation allows for policy optimization: we provide a general algorithm for optimizing causal interventions. We corroborate our theoretical results with numerical simulations.


\end{abstract}

\section{Introduction}\label{sec:intro}

The interface of causal inference and machine learning offers to ``deliver the right intervention, at the right time, to the right person''. An extensive line of research studies off-policy evaluation (OPE) and learning---evaluating the average causal outcomes under alternative personalized treatment assignment policies that differ from the treatment assignment which generated the data (and may have introduced confounding), 
so that one may optimize over the best such treatment rule~\citep{manski2004statistical,dudik2011doubly,zhao2012estimating,thomas2015high,athey2017efficient,kitagawa2018should,kallus2018policy}. Most of this work is based on the assumption that the appropriate decision criterion is an \textit{average} of individuals across a population. But various operational restrictions or settings imply that a decision-maker's utility is often not realized as an \textit{average} but rather as an \textit{output} of a downstream planning or decision-making problem. 

For example, in studying the effects of price incentives in a matching market (e.g., on a ride-share platform), a firm's revenue is not realized until it matches riders to drivers under certain constraints~\citep{mejia2021transparency, ma2021spatio}. While the marketplace may offer incentives to drive or accept rides and induce causal effects on individuals, 
the final utility is determined by the \emph{new} matches, taking into account operational constraints and structure. 
        
As another example, although job training (and personalized provision thereof) is commonly touted in causal inference and machine learning papers as a promising example for personalized treatment policy assignment \citep{athey2017efficient,kitagawa2018should,knaus2020heterogeneous}, labor economists voice a general concern that ``the possible existence of equilibrium effects on the efficiency of the programs seems quite real'' \citep[p.541]{crepon2016active}. The equilibrium concern is that personalized provision of job training may not lead to actual beneficial gains at \textit{the population} level due to externalities (substitution effects/congestion in matching) of the labor search process in a finite market. While impressive cluster-randomized trials have been deployed to assess these effects~\citep{crepon2013labor}, it would be useful if there exists a framework that can model the equilibrium effect and evaluate treatment policies directly based on available data of individual-level causal effects. 
In some settings, population-level impacts 
may be well-modeled as a downstream optimization response. The development of such a framework is our focus in the current paper.




We study a new framework for policy evaluation and optimization where there is a \textit{personalized treatment policy} on individual-level outcomes, and a \textit{policy-dependent} optimization response. The key difference between this model and previous work on off-policy evaluation and optimization is that: although treatments realize causal effects on \textit{individuals}, a treatment policy's value depends on a further downstream \textit{policy-dependent} optimization. We study how to evaluate different policies without bias (off-policy evaluation) and how to optimize for the optimal policy under this framework (policy optimization).


Our contributions are as follows: we first introduce the model of policy-dependent optimization response\footnote{For terminology, we use policy-dependent (optimization) response or downstream policy-dependent response to refer to the same concept.}, which we formulate as a nonconvex stochastic optimization problem. For off-policy evaluation, we develop a framework of \emph{policy-dependent optimization response}, decompose the bias that arises in this framework (``optimization bias'') and show how to control it via the design of estimators for the policy-dependent estimand. Finally, we provide a general algorithm for optimizing causal interventions. We corroborate the theoretical results with experimental comparisons.

\subsection{Related work}\label{sec:related_work}

We highlight the most relevant work from causal inference, off-policy evaluation, and optimization under uncertainty in the main text. We include additional or tangential discussion in \Cref{apx-relatedwork}.

There is an extensive literature on off-policy evaluation and optimization \citep[see, e.g., ][]{manski2004statistical,dudik2011doubly,zhao2012estimating, swaminathan2015batch}. Relative to this line of work, we focus on the introduction of a downstream decision response, arising for example from operational constraints. 

The case of {constrained policies} has been considered in the OPE literature. Our setting is conceptually different but overlaps in some application contexts.  Specifically, we decouple the downstream \textit{policy-dependent response}, i.e. over a similar constraint space, from treatment decisions that have causal effects. For example, \citet{bhattacharya2009inferring} studies the setting of ``roommate assignment'' with discrete types; i.e., perfect bipartite matching. A crucial difference is that in their setting, the causal treatment of interest \textit{is} the assignment decision to the other individual type; while in our setting the causal treatment only affects certain \textit{parameters} of the assignment decision, such as edge costs. We instantiate an analogous example in our framework to highlight our decoupled causal intervention and prediction decisions. Consider a setting with a causal intervention, such as a diversity information intervention affecting a student's probability of getting along with various types. Here the policy performs treatments on individual's diversity information, and the final assignment decision is policy-dependent response. Other work considers resource-budgeted allocation~\cite{kube2019allocating}, which is structured because of reformulation of thresholds \cite{lopez2020cost}. \citet{sun2021empirical} studies sharp asymptotics for the additional challenge of stochasticity in the budget.  

Some works that illustrate the embedding of causal effect estimates in optimization-based decision problems include \cite{rahmattalabi2022learning}, although their formulation is ultimately a mixed-integer optimization. 

In contrast to an extensive line of work on heterogeneous causal effect estimation \citep{shalit2017estimating,wager2018estimation,kunzel2019metalearners}, often crucially leveraging simpler structure of the treatment contrast rather than the conditional outcomes, in this work we require estimation of the latter due to the downstream optimization and distributional convergence for the perturbation method. In turn, combining causal outcome estimation with adjustments for optimization bias requires different properties of the estimation strategy, namely plug-in estimation of a modified regression model; we focus on estimators that modify the first-order conditions of a regression model to algebraically achieve an AIPW-type adjustment as discussed in~\citet{bang2005doubly}. See also \citet{scharfstein1999adjustingrejoinder,tran2019double, shi2019adapting,chernozhukov2021riesznet}.

This work focuses on the challenge of \textit{optimization bias} for policy evaluation introduced in our setting, for \textit{generic} linear optimization problems. This well-known challenge of in-sample optimization bias (``sample average approximation bias'') fundamentally demarcates the statistical regime of optimization under uncertainty from sample mean estimation~\citep{bayraksan2006assessing, shapiro2021lectures}. Recent work develops bagging, jackknife, perturbation and variance-corrected perturbation approaches for bias adjustment \citep{lam2018bounding,ito2018unbiased,kannan2020data,gupta2021debiasing}. We extend a perturbation method of \citet{ito2018unbiased} to the setting of nonlinear predictions.



\section{Preliminaries}\label{sec:prelim}

We first define the setting for off-policy evaluation with policy-dependent responses. We distinguish between the \textit{causal decision policy} $\pi$ and the \textit{downstream optimization response} $\z$. The causal decision policy $\pi$ intervenes on individual units, while the policy-dependent responses are solutions to a downstream optimization problem on the causal responses of all the units. 


\vspace{-2mm}
\looseness=-1 
\paragraph{1. Off-policy evaluation.}

We first describe the single time step off-policy policy evaluation and optimization problem \citep[see][for further context]{dudik2014doubly,hirano2020asymptotic}. Let covariates be $W \in \cW \subseteq \R^d$, binary treatment be $T\in \{0,1\}$\footnote{The extension to non-binary treatments is immediate.}, and potential outcomes be $c(T)$. Denote the covariates' distribution as $\cP$. Without loss of generality we consider lower is better for $c$; e.g. we minimize costs. We consider a setting of learning causal responses from a dataset of tuples $\mathcal{D}_1 = \{(W_i, T_i, c_i)\}_{i=1}^n$ where treatment is assigned randomly or in an observational setting; henceforth we call this the \textit{observational / experimental dataset}.

We let $\pi_t \colon \mathcal W \mapsto [0,1]$ denote a personalized policy mapping from covariates to a (probability of) treatment $t$. Later we will focus on parameterized policies, such as $\pi_t(w) = sigmoid(\pith^\top w)$ or policies that admit global enumeration. The goal of off-policy learning is to optimize the causal interventions (aka policies) by estimating average outcomes induced by any given policy. Throughout we will follow the convention that, a random variable $c(\pi_t)$ denotes $ c(\pi_t) =  c(t) \Ind(Z_t =t)$, 
where $Z_\pi \in \{0, 1\}$ is a Bernoulli random variable of policy assignment: $Z_\pi \sim \text{Bern}(\pi_1)$. Then, the (random) outcome for a given covariate with policy $\pi$ is:

{\centering
  $ \displaystyle
\begin{aligned}\textstyle
    c(\pi) =  \sum_t c(\pi_t) = \sum_t  c(t) \Ind(Z_\pi =t).
\end{aligned}
  $ 
\par}
The average treatment effect (ATE) of a policy $\pi(\cdot)$ is then $\E[c(\pi)]$, where the expectation is taken 
over the randomness of the covariates $W \sim \cP$, assignments induced by $\pi$, and $c$ conditional on realized treatment $t$ and covariates. 

\looseness=-1 
\paragraph{2. Policy-dependent responses.}

\textit{Policy-dependent optimization} solves a downstream stochastic linear optimization problem over a decision problem $x \in \cX \subseteq \mathbb{R}^m$ on the $m$ units given a causal intervention policy.  In particular, $m$ represents the dimension of the downstream decision problem. Relative to the downstream decision problem, causal outcomes may enter \textit{either} as uncertain objective coefficients (in $c$) \textit{or} constraint capacities (in $b$).\footnote{Throughout the text we focus on uncertainty in $c$ for notational clarity; strong duality implies the same results hold for uncertainty in the constraint right-hand-side, $b$. The decision is made conditionally on context information $W$ but prior to realizations of potential outcomes, aka a policy-dependent response. } 

\paragraph{Dimensionality of the responses.} 
We consider two different asymptotic regimes: an \textit{out-of-sample, fixed-dimension, fixed-$m$} regime and an \textit{in-sample, growing-dimension, growing-$n$} regime. 
We formalize the former regime, the main focus of the paper, in the following assumption. 
\begin{assumption}[Out-of-sample, fixed-dimension  regime]\label{asn-asymptotic regime} 
As $n\to\infty$, the dimension of the optimization problem $m$, given by a new draw of contexts $\cD_2=\{W_i\}_{1:m}$ remains finite. The decision-dependent response on $m$ units is measurable with respect to $\cD_2$. Let $c_i(\pi)\defeq \E[c(\pi(w) | w=W_i]$, 
we have that the policy value $v_\pi^\ast$ is:
\begin{equation}
  v_\pi^\ast = \textstyle \mathbb{E}[  \min_x\{  \summ c_i(\pi) x_i\colon {Ax \leq  b }\}]. \label{defn-oraclepolicydependentestimand} 
\end{equation}
\end{assumption}
\Cref{asn-asymptotic regime} defines our \textit{policy-dependent estimand} in this regime. The expectation is taken over the randomness of the policy $\pi$ and the randomness of the finite samples $\{w_i\}_{i=1}^m$. 
The main text focuses on statements in the regime of \Cref{asn-asymptotic regime}. Evaluating regret with respect to a fixed dimension is standard or implicit in the predictive optimization literature.\footnote{The predictive optimization literature instead views each dimension of the decision variable as a multivariate outcome; relative to that, our regime can be interpreted as the setting of a scalar-valued contextual response. }

\begin{assumption}[In-sample, growing-dimension regime]\label{asn-insample-regime} 
As $n\to\infty$, the limit of the objective function is an expectation over contexts.\footnote{Assume the constraint $b$ scales with $n$ in a meaningful problem-dependent way so that constraints are neither all slack nor infeasible in the limit.} The estimand is: 
\begin{equation}
   v_\pi^\ast = \textstyle \mathbb{E}[  \min_x\{  \E[ c(\pi) x]\colon {Ax \leq  b }\}]. 
\end{equation}
\end{assumption}
Recall that a policy maps from covariates to a (probability of) treatment. Assumption~\ref{asn-insample-regime} precisely takes an expectation over the two sources of randomness: the outer expectation is taken over the randomness of the policy, and the inner expectation is taken over the randomness of the covariates $w$.

The limiting object in the growing-dimension regime is a ``fluid limit'' or asymptotic regime: informally we assume a meaningfully constrained optimization in the limit. We instantiate our framework in the following example.

\begin{example}[{Min-cost bipartite matching}]\label{example:matching-part-1}
Our framework is precisely motivated by the practical challenges in causal inference tasks, where the problem of ``policy dependent" optimizations pops up repeatedly. For instance, 
for price incentives in a matching market (such as a rideshare platform), the revenue/welfare outcome is not realized until the riders and drivers are matched under constraints. As another example, consider a manager wants to assign agents to different jobs, and assigning an agent to a job is associated with some cost. Our goal is to assign each agent to at most one job such that the overall cost is minimized. To incentivize the workers to complete the jobs, the company might want to provide some bonus to the agents. However, the overall efficiency and total payments are not realized until all the assignments are determined.

The above type of application can be modeled as a min-cost bipartite matching problem, which is well known to have a totally unimodular linear relaxation. Clearly, the agents (or riders) and the jobs (or passenger requests) form the two sides of nodes for the matching. The edge costs in the matching stand for the cost or payment for an agent to complete that job. A treatment ($T=1$) serves as intervention on the edge costs for that agent, and the covariates $W$ could be any observable features of the agents, such as preferences, demographic information, etc. 
Given any allocation rule of the bonuses, the manager faces a downstream min-cost bipartite matching:
\begin{align}\label{eq:matching-example-opt-x}\textstyle
\begin{split}\textstyle\min_{x\in \{0,1\}^{|\cE|}} 
\left\{ \sum_{e \in \cE} c_e(\pi) x_e \;\colon 
\sum_{e \in \cN(i)} x_e = 1, \forall i \in \cV \right\}.
\end{split}
\end{align}
Here $\cN(i)$ is the set of all edges contains node $i$, the $c_e$ are the edge costs, and $x = \{x_e\}_{e\in \cE}$ represents the matching where $x_e = 1$ means that edge $e$ is selected\footnote{In the later analysis we use the linear relaxation with $x_e \in [0,1]$ (continuous interval). For bipartite matching because of \textit{total unimodularity} the linear relaxation is tight and equivalent to integral formulation.}.
\end{example}
In \Cref{sec:decdependentclassifier} we include an additional example of predictive risk optimization, beyond linear optimization, which requires a different estimation strategy.

\vspace{-2mm}
\looseness=-1 
\paragraph{3. Policy optimization with policy-dependent responses.}

Putting together the pieces of the previous subsections, the off-policy optimization over candidate policies $\pi \in \Pi$ is:
\begin{equation}
\label{eqn-framework}
\textstyle  \min_{\pi \in \Pi} \min_{x \in \cX}
  \left\{ \summ c_i(\pi) x_i \colon {Ax \leq b }
  \right\},
\end{equation}
where $m$ represents the dimension of the decision problem (e.g., the number of edges in \Cref{example:matching-part-1}), and $x$ denotes the whole response vector $\{x_i\}_{i\in [m]}$. 

We illustrate this framework by revisiting our examples.

\begin{example}[Policy optimization for \Cref{example:matching-part-1}, min-cost matching]
In the min-cost bipartite matching example, the optimal assignments with a given policy $\pi$ can be solved via the linear program in \Cref{eq:matching-example-opt-x}. Suppose that we want to find the best intervention policy which gives the lowest matching cost. Then, the policy optimization problem is:
\begin{align*}
\textstyle
\min_{\pi \in \Pi}\min_{x\in \cX}\; \textstyle \bigg\{\sum_{e \in \cE} c_e(\pi) x_e: \sum_{e \in \cN(i)} x_e = 1, \forall i; x_e \geq 0, \forall e\bigg\},
\end{align*}
where $\Pi$ denotes the set of all policies that are of interest.
\end{example}

\section{Problem Description: Optimization Bias}\label{sec:prob-description}

We focus on off-policy evaluation in view of the 
downstream optimization over the decision variables $x=\{x_i\}_{i \in [m]}$. We first discuss \textit{plug-in} estimation approaches without causal adjustment to introduce
the challenge of optimization bias in this regime. We then discuss causal estimation in \Cref{sec:estimation}. 

\looseness=-1 
\paragraph{From estimation bias to optimization bias. } 
Denote ${\mu_t(w) = \E[c(t)\mid W=w]}$ as the conditional outcome mean of the population with treatment $t$ and covariates $w$. We consider ``predict-then-optimize'' approaches which learn some ${\estmu_t(w) = \E[c\mid W=w, T=t]}$ and optimize with respect to it, so that our estimator is: 

{\centering
  $ \displaystyle
\begin{aligned}\textstyle
 \hat v_{\pi} =\textstyle  \min_{x \in \cX}
  \left\{ \summ \sumt   \pi_t(w_i)  \hat\mu_t(w_i) x_i \colon {Ax \leq b} \right\}.\end{aligned}
  $ 
\par}
\looseness=-1
Note that due to the estimation and minimization step,  $\hat v_\pi$ is not an unbiased estimator for $v^\ast_{\pi}$. Define the overall error of $\hat v_\pi$ with respect to the target estimand of \cref{defn-oraclepolicydependentestimand} as: $ \text{err} = v^\ast_{\pi} - \E\left[\estmuv_{\pi}\right]$. 
We decompose the overall error 
into two parts: the estimation bias of the plug-in estimator, and the optimization bias.
Denote $\oraclemuv_{\pi}$, the best-in-class feasible estimate using
the true conditional expectations $\mu_t^*$: 

{\centering
  $ \displaystyle
\begin{aligned}\textstyle
 \oraclemuv_{\pi} =\textstyle  \min_{x \in \cX}
  \left\{ \summ \sumt   \pi_t(w_i)  \mu_t(w_i) x_i \colon {Ax \leq b} \right\}.\end{aligned}
  $ 
\par}
Then, the estimation and optimization biases are: (by triangle inequality, ${|\text{err}| \leq |    \text{bias}_{\text{est}}| + |    \text{bias}_{\text{opt}}|}$)
\begin{align*}
    \text{bias}_{\text{est}} = \E[\estmuv_{\pi}] - \E[\oraclemuv_{\pi}], \qquad 
    \text{bias}_{\text{opt}} = v_{\pi}^\ast - \E[\oraclemuv_{\pi}]. 
\end{align*}



\begin{table*}[!t]
\centering
\begin{adjustbox}{max width=0.95\textwidth}
\begin{tabular}{l|ll|ll}
\toprule
         & \multicolumn{2}{l|}{
         \textbf{Out of sample, fixed $\mathbf m$} (\Cref{asn-asymptotic regime})}                       & \multicolumn{2}{l}{\textbf{In-sample, growing $\vn$} (\Cref{asn-insample-regime}, \Cref{apx-alt-asymptotic-regime})    }                                                                         \\
         & \color{eval}{Evaluation}                              & \color{optim}{Policy  optimization}   &\color{eval}{Evaluation} & \color{optim}{Policy optimization}        \\ \cmidrule{1-3} \cmidrule{4-5} 
\textbf{AIPW} & \multicolumn{2}{c|}{\color{neutralvals}{N/A}}                                                                        & \multicolumn{1}{l|}{\color{eval}{Sample splitting (finite VC-dim $x$)} }                                                                                                          & \multirow{3}{*}{\begin{tabular}[c]{@{}l@{}}\\\color{optim}{Uniform generalization} \\\color{optim}{requires  problem-dependent}\\ \color{optim}{structure (finite VC-dim $x$)}\end{tabular}} \\ \cmidrule{2-4}
\textbf{WDM}      & \multicolumn{1}{l|}{
\color{eval}{Perturbation method}  } & \multirow{2}{*}{\begin{tabular}[c]{@{}l@{}} \color{optim}{Uniform generalization} \\\color{optim}{from out-of-sample} \\\color{optim}{risk bounds}\end{tabular}} & \multicolumn{1}{l|}{\color{eval}{Perturbation}}                                                                                          &             \\ \cmidrule{2-2} \cmidrule{4-4}
\textbf{GRDR}     & \multicolumn{1}{l|}{\color{eval}{Perturbation method}} &      & \multicolumn{1}{l|}{\begin{tabular}[c]{@{}l@{}}\color{eval}{Perturbation}\\ \color{eval}{Doubly-robust estimation} 
\end{tabular}} &                                         \\ \bottomrule                             
\end{tabular}\end{adjustbox}
	\caption{Summary of regimes and estimation properties. The main text provides methods for \Cref{asn-asymptotic regime}. Additional structural restrictions permit extensions for \Cref{asn-insample-regime}. }
	\label{table:result} \vspace{-2mm}
\end{table*}

\looseness=-1
\paragraph{In-sample estimation bias due to optimization.} 

It is well known that in-sample estimation of the value of optimization problems is biased; e.g.,
$ \estmuv
  $
is a biased estimate for the true objective value $v_\pi^*$ due to optimization.
\citet{ito2018unbiased} studies a bias correction for affine linear objectives with an unbiased estimate of a parameter $\theta$. To understand the source of the bias due to optimization, observe that clearly 
$\summ   \mu_t(w_i) x_i \geq \min_x \summ   \mu_t(w_i) x_i$. 
The inequality remains valid when evaluating expectations over training datasets so that $\textstyle {\E[ \summ   \mu_t(w_i) x_i] \geq \E[ \min_x \summ   \mu_t(w_i) x_i]}.$
Noting that the RHS is the true objective $ v_\pi^*
$, we obtain in general the well-known optimistic bias, that 
$ \E[\oraclemuv_{\pi}] \geq v_\pi^*.$ In the policy evaluation setting, our estimates converge to the LHS, $\oraclemuv_\pi$, so that our estimator $\hat v_\pi$ is in general a \textit{biased} estimate of the decision-dependent policy value even if we obtain \textit{unbiased} estimates of the cost coefficient. 




\section{Causal Estimation with Policy-Dependent Responses}\label{sec:estimation}

In this section we present an estimation approach 
building upon a perturbation method that adjusts for the aforementioned optimization bias. We summarize tradeoffs among estimation strategies in different regimes in \Cref{table:result} and possible extensions and additional structure in \Cref{apx-alt-asymptotic-regime}.

\subsection{Estimating causal effects: estimation bias}\label{sec-estimation}



\begin{assumption}[Ignorability, overlap, SUTVA]\label{asn-ignorability} For all $t$, ${c(t)\indep T \mid W}$. The evaluation policy is absolutely continuous with respect to treatment probabilities in the training dataset. Assume the stable unit treatment value assumption. 
\end{assumption}

\paragraph{Confounding-adjusted plug-in estimators.} 
\looseness=-1
In general, plug-in estimation of $\estmu_t(W)$ does \textit{not} admit unbiased predictions because of selection bias and model misspecification. Existing importance-sampling based estimators, e.g. the inverse propensity weighting (IPW) estimator and the doubly-robust augmented inverse probability weighting (AIPW) uses the propensity score to adjust confounding, under \Cref{asn-ignorability}. Note importance sampling cannot \textit{directly} be applied in our main regime of interest with out-of-sample evaluation as in \Cref{asn-asymptotic regime}, see \Cref{apx-estimation} for a detailed overview.

We depart from previous work in off-policy evaluation, in view of the optimization bias adjustment (detailed in the next section), and study estimation methods that are \textit{plug-in estimates} for OPE: $\mathbb{E}[c(\pi)]=\sum_{t} \mathbb{E}[\pi_t(W) \estmu_t(W)]$, for some outcome model $\estmu_t$ that is confounding-adjusted.

Note that IPW/AIPW-type estimators cannot be applied in the out-of-sample regime of \Cref{asn-asymptotic regime}. However, we may obtain out-of-sample risk bounds on the decision regret in this regime by virtue of out-of-sample generalization risk bounds on the generated regressors. We include more detailed discussion in Appendix~\ref{sec:appendix-ipw}.

 
\paragraph{Weighted direct method (WDM).} 
Outcome regression, learning $\estmu_t(W) = \E[c \mid T=t, W]$ directly from $\cD_1$, is sometimes called the \textit{direct method}. However, when $\estmu$ is a misspecified regression model such a method incurs bias. 
Nonetheless, re-weighting the estimation $\estmu$ (maximum likelihood, empirical risk minimization) by the inverse probability weights $1/e$ is known to adjust for the covariate shift; by a similar argument as that of \cite{shimodaira2000improving,cao2009improving,wang2019batch}. We call this approach \textit{weighted direct method} ($\WDM$), which solves: 
\begin{equation}
   \textstyle \estmu_t^{\WDM} \in \arg\min_\mu \E\left[ \frac{\mathbb{I}(T=t)}{e_t(W)} (c-\mu_t(W))^2 \right].
\end{equation}
\paragraph{Doubly-robust direct method (GRDR).}
We also consider an approach that achieves doubly-robust estimation of the treatment-effect due to \citet{bang2005doubly}. \citep[See also][]{scharfstein1999adjustingrejoinder,tran2019double}.  This approach  has been used for CATE estimation \cite{shi2019adapting,chernozhukov2021riesznet}. The inverse propensity score reweighted treatment indicator is added as a covariate in the model, inducing coefficients $\epsilon_0,\epsilon_1$. Define 
\begin{equation*}
\textstyle \estmu^{\GRDR} = \mu(W) + \epsilon_1 (T/e_1(W)) + \epsilon_0 ((1-T)/ e_0(W)).
\end{equation*} 
Optimizing over $\estmu$ by (nonlinear) least-squares yields the following first-order optimality conditions for ${\theta^{\GRDR}=[\bar{\theta}, \epsilon_1, \epsilon_0]}$: 
\begin{align}
&\E[ (c-\estmu) \nabla_\theta \estmu ]=0,\E[ (c-\estmu) (T/ e_1(W)) ]=0,
 \E[ (c-\estmu) ((1-T)/ e_0(W)) ]=0.  \label{eqn-esteqn-grdr}
\end{align}
\citet{bang2005doubly} show that the first-order optimality conditions ensure that plug-in estimation of an average treatment effect with the model is equivalent to AIPW, hence doubly-robust. Because it is designed primarily for estimation of the ATE, its use as an outcome predictor is more speculative. Although one can verify that its output is covariate-conditionally equivalent to CATE in expectation, and one can use this fact to again regress upon the pseudooutcomes, this final procedure would require re-verifying asymptotic convergence; we don't outline those arguments here. We include further discussion on the different estimation interpretations of GRDR in the two regimes in Appendix~\ref{sec:appendix-ipw}.






\subsection{Estimating the decision-dependent estimand}

Our procedure is adapted from the perturbation method of \citet{ito2018unbiased} which we describe here for completeness; we extend it from linear to nonlinear predictors. 
The method of \citet{ito2018unbiased} focuses on one parameter that we denote $\perturbparam=[\theta,\gamma]$, where we assume as outlined in \Cref{eqn-estasn-momentcondition} that it encompasses parameters of the outcome and propensity model (respectively). Define the policy-induced outcome model, ${\mu_\pi(w) = \sum_t \pi_t(w) \mu_t(w)}$,  the estimation error $\delta=\hat\perturbparam-\perturbparam^*$, and the (parametrized) optimal solution at a given predictive model $x(\perturbparam)$. The perturbation method is motivated by a finite-difference approximation to the optimization bias induced by estimation error $\delta$. Define the auxiliary functions given a scalar $\epsilon$ parametrizing the direction of $\delta$: \begin{align*}
    \eta(\pathg) &= \textstyle\mathbb{E}_\delta \left[ \summ \z(\perturbparam^*+\pathg \delta) \pi(W;\tht^*)
    \right],\;\; \phi(\pathg) = \textstyle\mathbb{E}_\delta \left[ \summ \z(\perturbparam^*+\pathg \delta) 
    \pi(W;\tht^*+\pathg \delta)
    \right].
\end{align*}
We require regularity conditions for derivatives of these functions to exist: 
\begin{assumption}[Perturbation method assumptions]\label{asn-perturbationregularity} 
(i) The optimal solution $\z(\perturbparam)$ is unique.
    (ii) $\hat\perturbparam$ is an unbiased estimator of $\perturbparam^*$. 
\end{assumption}
We generalize Prop. 3 of \citet{ito2018unbiased} for nonlinear models.
\begin{proposition}[
]\label{prop-pathderiv}
We have $\eta(\pathg) =\phi(\pathg) -\pathg \phi'(\pathg) + O(\pathg^2)$.
\end{proposition}
The plug-in estimated optimal value $\hat v_\pi$
unbiasedly estimates $\phi(1)$. Note $\phi'(1)$ is equivalent to the value of the bias. The perturbation method estimates $\phi'(1)$ by $(\phi(1+h)-\phi(1))/h$ for some small $h$. 

It remains to estimate $\phi(1+h)$. First we obtain $s$ samples of the perturbed parameter ${\hat\perturbparam_h=\perturbparam^*+(1+h)\delta}$, denoted as $\{ \hat \perturbparam_h^{(j)} \}_{j=1}^s$. 
Each replicate of $\hat\perturbparam^{(j)}$ leads to an optimization estimate ${\hat v^{(j)} = \min_{\z} \summ \z_i 
\hat\mu_\pi(w_i,\hat\perturbparam_h^{(j)})}
$. The debiased estimator is: 
\begin{equation*}
     \rho_h = \textstyle \hat v^{(0)} - \frac{1}{h}(\hat v^{(0)} - \frac{1}{s} \sum_{j=1}^s \hat v^{(j)})
\end{equation*}
Our \Cref{prop-pathderiv} then implies asymptotic unbiasedness (cf.\ Prop. 4 of \citet{ito2018unbiased}) so that ${\ts \lim_{h\to 0} \mathbb{E}[\rho_h] = \mathbb{E} [ \min_{\z} \sum_{i=1}^m \z_i 
\hat\mu_\pi(w_i,\perturbparam^*)
]}$. We summarize the method in \Cref{alg:perturbation-method}.

\begin{algorithm}[!t]
\caption{Perturbation method, Alg. 2 of  \cite{ito2018unbiased})}\label{alg:perturbation-method}
\begin{algorithmic}[1]
\INPUT Estimation strategy $\genest \in \{\WDM, \GRDR\}$; $h$: finite different parameter; $\pi$: policy. 
\item Estimate $\hat\perturbparam_\genest=[\hat\parammu_\genest, \hat\parame_\genest]$ for $\hat\mu^{\genest}$ from $\cD_1$
\STATE $\hat v^{(0)} \gets \min_{\z \in \mathcal X} 
\summ\z_i  \sumt \pi_t(w_i) \hat\mu_t^\genest(w_i; \hat\perturbparam_\genest)$   

\STATE Generate 
$\{ \perturbparam_\genest^{(j)} \}_{j=1}^s$:
 if by parametric bootstrap, learn
         $\hat{\perturbparam}_\genest^{(j)} $ from $\frac{N}{(1+h)^2}$ samples randomly chosen from $\mathcal{D}_1$ with replacement. 

Otherwise if using $\hat\Sigma$, estimator of asymptotic variance of $\perturbparam$, approximate the distribution of $\perturbparam^* + (1+h)\delta$. Add $\hat \perturbparam$ to $\hat\theta$ where $\hat\delta \sim N(0, \frac{(1+h)^2-1}{N} \hat\Sigma).$ Then set $\hat\perturbparam_\genest^{(j)}=\hat \perturbparam + \hat{\delta}_j.$
\FOR{$j=1, \dots, S:$}
\STATE{$\hat v^{(j)} \gets  \min_{\z \in \mathcal{X}} \summ  \z_i
\sumt \pi_t(w_i) \hat\mu_t^\genest(w_i; \hat\perturbparam_\genest^{(j)}) $. }
\ENDFOR
\STATE  Output $\rho_h =  \hat v_0 - \frac{1}{h}(\hat v^{(0)} - \frac{1}{s} \sum_{j=1}^s \hat v^{(j)})$.
\\[1ex]
\end{algorithmic}
\end{algorithm}




\paragraph{Asymptotic variance of estimation methods.}

We discuss the asymptotic variance of the \textit{weighted direct method} and $\GRDR$  via classical asymptotic analysis of \textit{generated regressors} (specifically, stacked estimation equations of GMM) \citep{newey1994large}. We summarize this framework in \Cref{apx-estimation-preliminaries} for completeness and include the main result here that we invoke.\footnote{Asymptotic normality of these approaches is taken as given in \citet{cao2009improving,bang2005doubly} and so we include these statements for completeness. For exposition and context of Donsker-type conditions in semiparametric inference, see \citet{kennedy2016semiparametric} or other references.} 

\begin{assumption}[Estimators via GMM with generated regressors]\label{eqn-estasn-momentcondition} 
\looseness=-1
Suppose the propensity score $e$ and outcome model $\mu$ are indexed by true parameters $\gamma^*, \theta^*$ that solve the respective estimating equations $  \mathbb E [ h(W, \gamma^*) ] = 0,\quad  \mathbb E [ g(W, \theta^*, \gamma^*)  ] = 0.$
The functions $e_t(w),\mu_t(w)$ are in a Donsker class.
\end{assumption}
\begin{remark}[Strength of assumptions]
\Cref{alg:perturbation-method} requires both unbiased and asymptotically normal predictions---stronger conditions than merely inference on the ATE. The Donsker assumption preserves asymptotic normality with generated regressors. The framework allows for nonparametric estimation via linear sieves (but not some high-dimensional regimes; see \citet{ackerberg2012practical}).  
\end{remark}



\begin{theorem}[Thm. 6.1, eq. 6.12 of \citet{newey1994large}]
Suppose \Cref{eqn-estasn-momentcondition} holds.
Let $\hat G_\alpha, \hat G_\theta, \hat H$ denote the Jacobian matrices of partial derivatives of the moment conditions $g,h$ with respect to the respective parameters, i.e. ${\hat G_\gamma = \avgn \nabla_\gamma g(w_i, \hat \theta, \hat \gamma)}$.  Let ${\hat V_\gamma = (\hat{H}^{-1} \hat h_i)(\hat{H}^{-1} \hat h_i)^\top}$. Then an estimator of the asymptotic variance is: 

\resizebox{3.35in}{!}{
$\textstyle
\hat V_\theta = 
\hat G_\theta^{-1}
\left( \avgn \hat g_i  \hat g_i^\top
\right)
(\hat G_\theta^{-1})^\top
+ \hat G_\theta^{-1} \hat G_\gamma \hat V_\gamma \hat G_\gamma^\top (\hat{G}_\theta^{-1})^\top.$ 
}
\end{theorem}
Since $\hat V_\gamma$ depends only on the specification of the propensity score, to completely specify the asymptotic variance for the above formula we state the mixed terms $\hat G_\gamma,\hat G_\theta$. 
\begin{proposition}[Asymptotic normality of $\WDM$]\label{prop-avar-weighteddm}
 Let $e_t(w), \mu_t(w)$ satisfy \cref{eqn-estasn-momentcondition} with the moment condition ${g_t(W,\parammu,\parame)=e_t(W;\parame)^{-1} (c- \mu_t(W;\parammu))^2}$ and $g = [g_0,g_1]$. Then 
    
    {\centering
  $ \displaystyle
\begin{aligned}
   \textstyle  \hat{G}_\gamma= 
   \begin{bmatrix}
       \mathbb{E}_n
[  2T (c - \mu(W;\parammu))\frac{\partial }{\partial \parammu}(e_1^{-1}(W, \parame))  \frac{\partial \mu}{\partial \parammu} ] \\
       \mathbb{E}_n
[  2(1-T)(c - \mu(W;\parammu))\frac{\partial }{\partial \parammu}(e_0^{-1}(W, \parame))  \frac{\partial \mu}{\partial \parammu} ] 
   \end{bmatrix}.
\end{aligned} 
  $ 
\par}
\end{proposition}

These formulas are generally computable from standard output of optimization solvers for nonlinear least squares: gradients and Hessians. In practice, using the parametric bootstrap may be simpler at a higher computational cost. 

\begin{algorithm}[!t]
\caption{Subgradient method for policy optimization}\label{alg:subgradient}
\begin{algorithmic}[1]
\STATE \textbf{Input:} step size $\eta$, linear objective function $f$.\\[1ex]
\FOR{$j = 1, 2, \cdots$}
\STATE At $\pith^k$,
obtain a subgradient in subdifferential 
$\textstyle  \cS^\ast(\pi_{\pith}^k) = \{x^\ast: f(x^\ast; \pi_{\pith}^k) = \min_{x} f(x; \pi_{\pith}^k)\}$  \label{eqn-step-evaluation}
\STATE Compute subgradient    
$\textstyle   \nabla_\pith (\min_x f(x; \pi_{\pith}^k) ) \gets 
 \nabla_\pith f(x^\ast; \pi_{\pith})$
    
\STATE Update subgradient step: ${\pith^{k+1} \gets \pith^k - \eta \nabla_\pith \left(\min_x f(x; \pi_{\pith}^k) \right)}$
\ENDFOR
\\[1ex] 
\end{algorithmic}
\end{algorithm}

\subsection{Optimizing Causal Interventions}\label{sec:opt}


Algorithm~\ref{alg:perturbation-method} provides estimation for a fixed policy. We now discuss how to optimize over policies; e.g., implementing the outer optimization over policies $\min_{\pi \in \Pi}$ in \cref{eqn-framework}. We focus on the case where the policy $\pi_t(w)$ is parametrized by and differentiable in a parameter $\pith\in \Psi$. For example, for the logistic policy parameterization, $\pi_t(w) = sigmoid(\pith_t^\top w)$. We consider a robust subgradient method, based on Danskin's theorem, detailed in \Cref{alg:subgradient}. Such an approach is a common heuristic used in adversarial machine learning.

We solve the inner optimization problem to full optimality in line 3 and take (sub)gradient steps for the outer optimization. We evaluate (sub)gradients of the inner optimization solution in line 3 by evaluating the gradient of the objective with respect to $\pith$, fixing the inner optimization variable $\z^*$.
Danskin's theorem implies that $\nabla_\pith$ is a subgradient \citep{danskin1966theory}. The inner minimization can be solved via a linear optimization oracle for any fixed choice of policy. This use of the linear optimization oracle can be beneficial when special problem structures, such as matching and network flows, may also admit readily-available algorithmic solutions to full optimization. 

The perturbation method is compatible with our optimization procedure because
the bias-adjusted perturbation estimated from \Cref{alg:perturbation-method} is affine in the optimization problems corresponding to each parameter replicate. Hence, run \Cref{alg:perturbation-method} with an expanded linear objective over the $s$-product space $x'\in \mathcal{X}^s$ where $\textstyle f(\tilde x,\pi) = \textstyle \hat v^{(0)}_\pi(\tilde{x}_0) - \frac{1}{h}(\hat v^{(0)}_\pi(\tilde{x}_0) - \frac{1}{s} \sum_{j=1}^s \hat v^{(j)}_\pi(\tilde{x}_j)).$ 

So,
re-optimize ${\tilde{x}^*_j \in\arg\min_{x\in\mathcal X} \summ x_i \hat\mu_\pi^\genest(w_i;\hat\perturbparam_\genest^{(j)})}$ and apply Danskin's theorem to each optimization problem in the sum over $\hat v^{(j)}_\pi$ comprising $f(x',\pi)$. In fact, though adversarial machine learning focuses on min-max rather than our min-min optimization problem, this particular approach is simply subgradient descent on a nonconvex function (the solution to the inner optimization).

\section{Experimental Evaluation}\label{sec:experiments}

\begin{figure}
\centering
\vspace{-1cm}
    \subfloat[Effect of $W$]{\includegraphics[height=0.28\linewidth]{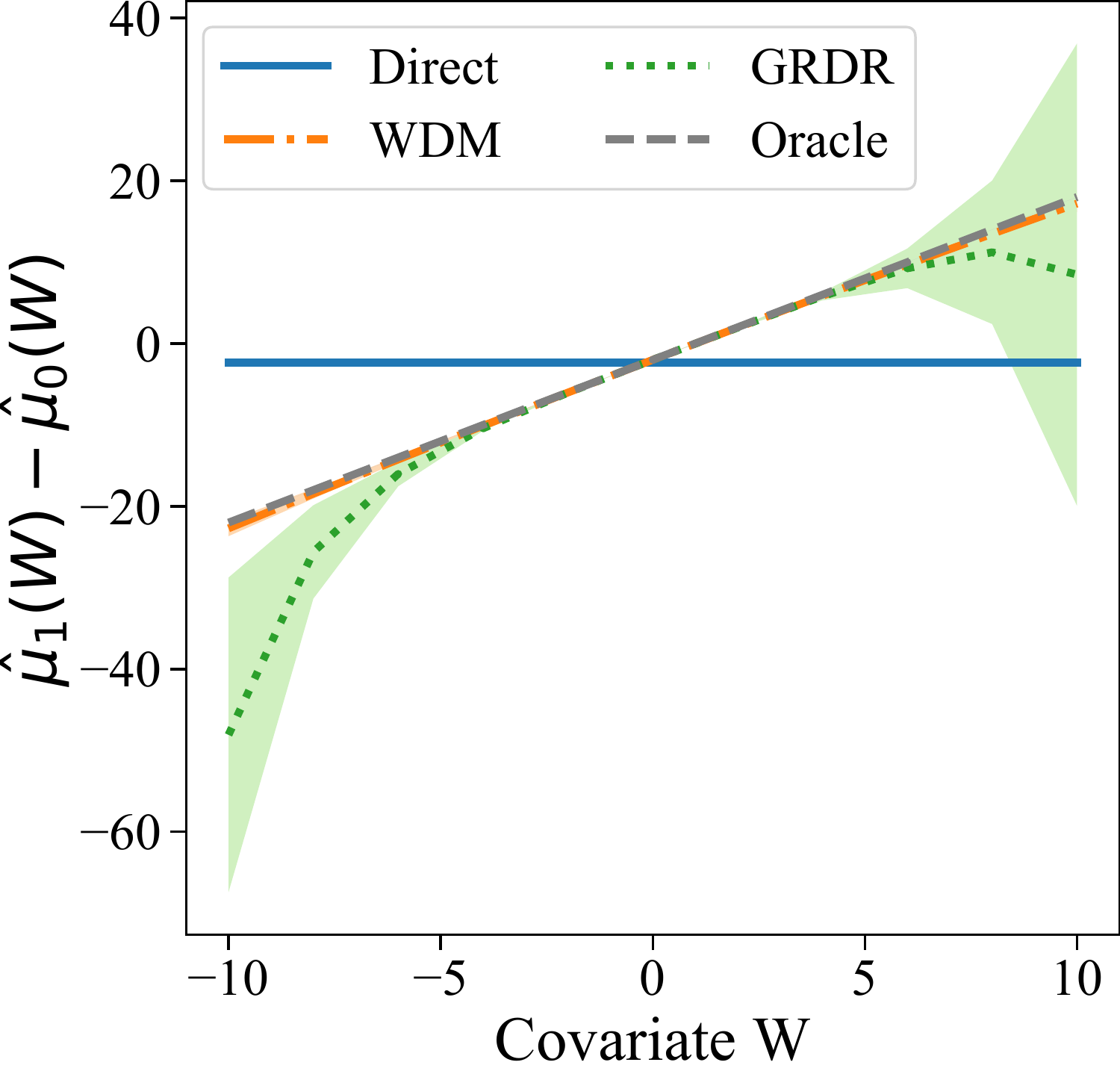}\label{fig:in-sample-CATE-a}} \quad \quad
    \subfloat[{\centering Effect of dataset size}]{\includegraphics[height=0.28\linewidth]{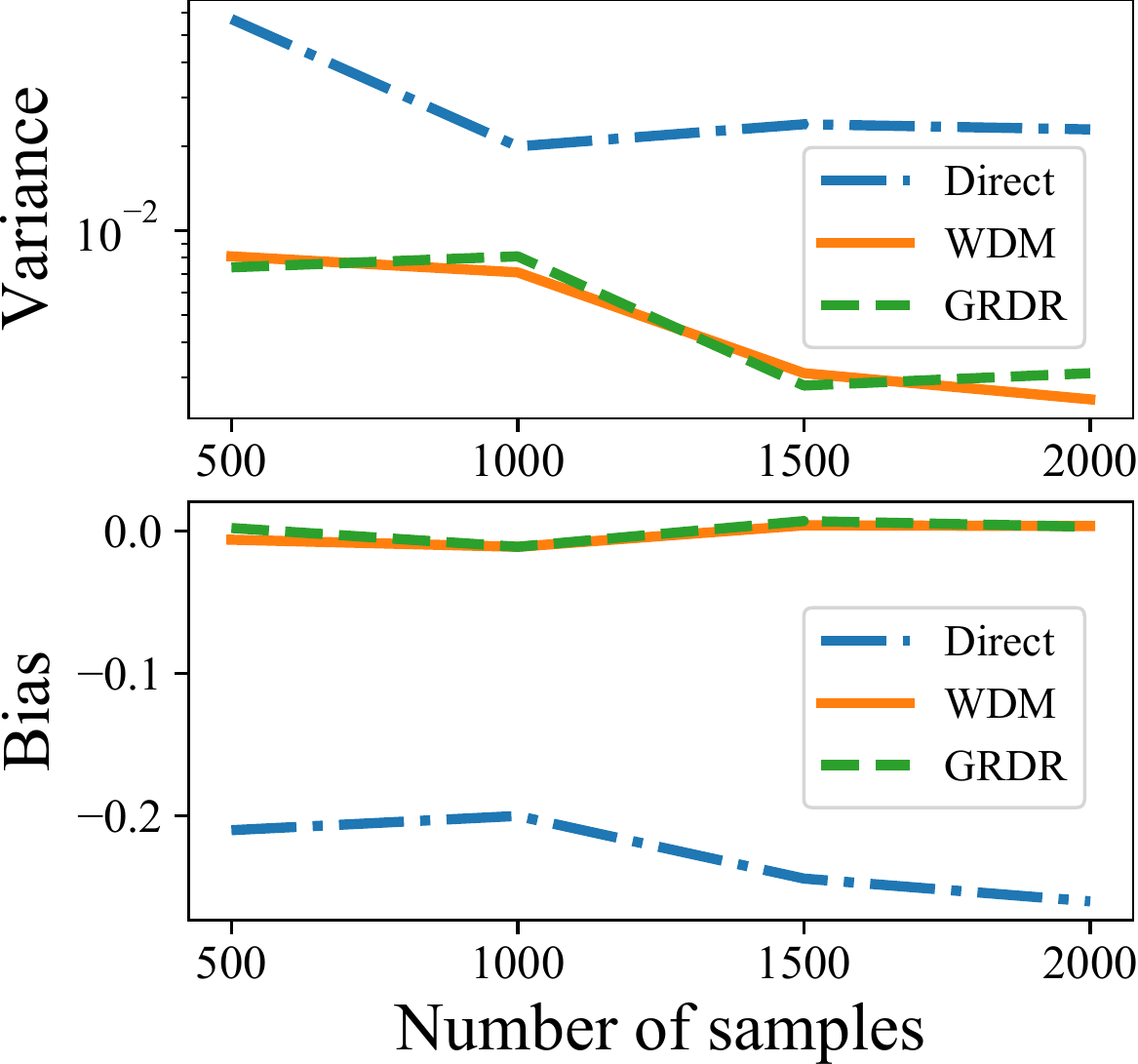}\label{fig:in-sample-CATE-b}}
     \caption{\emph{(In-sample estimation of $\hat\mu_1(W) - \hat\mu_0(W)$, with model mis-specification)}.  Comparison of direct / $\WDM$ / $\GRDR$ to the oracle. (a) Conditional estimation error averaged over ten random train sets; shaded area indicates std. error. (b) Bias / variance comparison with varying training data size. }
\end{figure}

Since real data suitable for both policy evaluation and downstream optimization is unavailable, we focus on synthetic data and downstream bipartite matching. 
We first illustrate estimation properties of the different approaches 
before showing the improvement obtained via policy optimization. Though we are not aware of prior approaches that are directly comparable for optimizing causal policy with a downstream optimization-dependent response, 
we include more comparisons to nonparametric estimators (e.g. causal forests~\citep{wager2018estimation}), and full implementation details. All code will be published.

\looseness=-1
\textbf{1. Causal effect estimation.} First, we investigate and illustrate the properties of different estimators. 
We generated dataset ${\cD_1 = \{(W, T, c)\}}$ with covariate $W \sim \cN(0,1)$, confounded treatment $T$, and outcome $c$. 
Treatment is drawn with probability ${\pi_t^b(W) = ({\mathrm{1} + e^{-\pith_1 W + \pith_2}} })^{-1}$, $\pith_1 = \pith_2 = 0.5$. The true outcome model is given by a degree-2 polynomial,\footnote{If not stated otherwise we spread the coefficients as ${poly_{\theta}(t, w) = (1, w, t, w^2, wt, t^2)\cdot([5,1,-1,2,2,-1])^\top}$. Additional supporting experiments under other nonlinear data-generating processes are in Appendix~\ref{app:exp}.} ${c_t(w) = poly_{\theta}(t, w) + \epsilon}$, where 
$\epsilon \sim \cN(0,1)$. .
In Figure~\ref{fig:in-sample-CATE-a} and \ref{fig:in-sample-CATE-b}, we illustrate the (covariate-conditional) estimation error of the three estimators. In the mis-specified setting that induces confounding, the outcome model is a vanilla linear regression over $W$ without the polynomial expansion. The direct method results in more bias under mis-specification, while $\WDM$ and $\GRDR$ are robust as expected.

\looseness=-1
\textbf{2. Policy evaluation.} We compare the perturbation method (Algorithm~\ref{alg:perturbation-method}) with three different estimators (direct, $\WDM$, and $\GRDR$). In both the well-specified / mis-specified model setting, we evaluate the mean-squared-error (MSE) of the estimated policy value with the three estimators, where the MSE is computed with regard to the ground-truth outcome model. Training data size $n$ increases from 500 to 2000 samples. We scaled the MSE down by the number of edges (a constant) and computed the MSE in terms of the averaged cost per edge in the matching.

For the policy-dependent optimization, we evaluate a min-cost bipartite matching problem, where the causal policy intervene on the edge costs (as detailed in \Cref{example:matching-part-1}). Specifically, the bipartite graph contains $m=500$ left side nodes $W_1, \cdots, W_m$, and $m'=300$ right side nodes. The policy $\pi_t$ applies treatments to the left side nodes and the outcome is the edge cost of edges with that node. While we grow the training data size, we fixed $m, m'$ (with $m>m'$) and evaluate over ten random draw of train/test data for each value of $n$. Figure~\ref{fig:main-pol-eval} plots the results. 
When there is mis-specification, even a large training dataset cannot bring bias correction for the direct method, where both $\WDM$ and $\GRDR$ enjoy smaller and decreasing MSE. 

We also conduct an ablation study for the corresponding performance in the mis-specified setting (i.e., no bootstrapping in Alg.~\ref{alg:perturbation-method}). Results indicate that the perturbation method is helpful for MSE reduction for both $\WDM$ and $\GRDR$. We further conduct evaluations with different bootstrap replicates' sizes, and the above conclusions remain robust for different replicate sizes (additional results in Appendix~\ref{app:exp}).

\begin{figure}
\begin{minipage}[c]{0.45\linewidth}
 \subfloat[{\centering well-specified models}]{\includegraphics[height=0.45
     \linewidth]{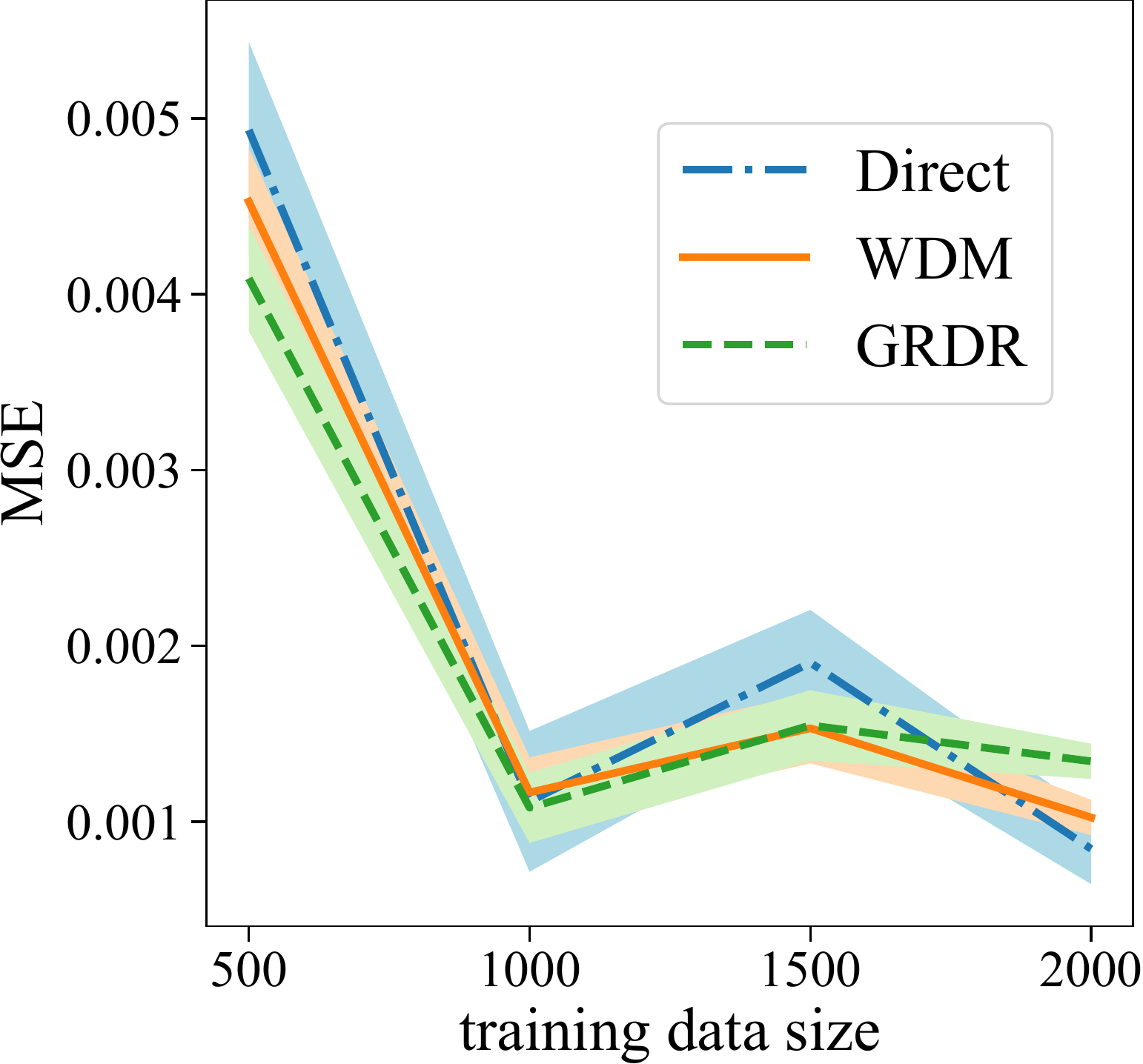}\label{fig:pol-eval-well}}\hfill
    \subfloat[{\centering with  mis-specification}]{\includegraphics[height=0.45\linewidth]{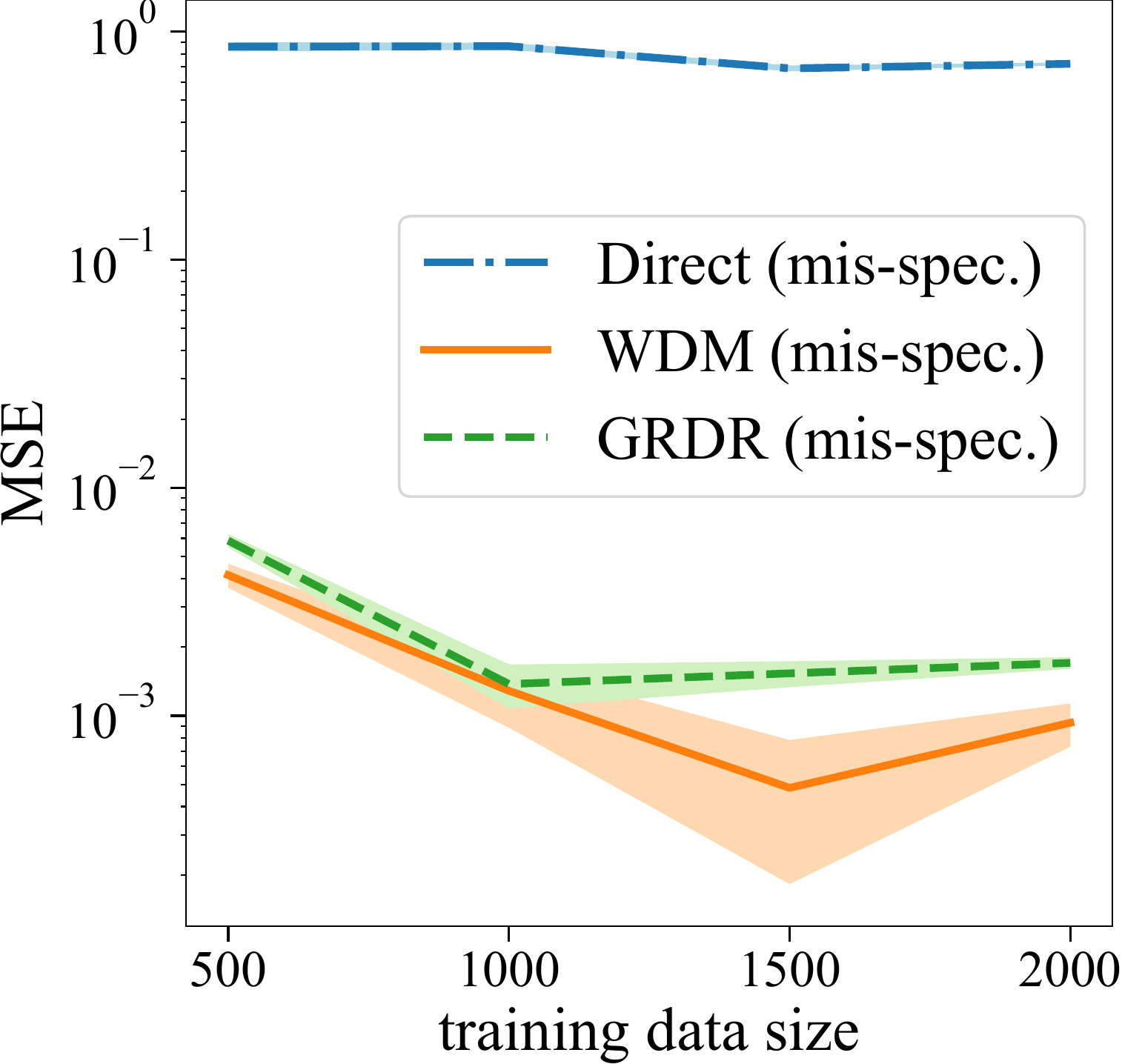}\label{fig:pol-eval-mis}}
     \caption{\emph{(Policy evaluation via perturbation method (Algorithm~\ref{alg:perturbation-method})}.   Comparison of direct / $\WDM$ / $\GRDR$ estimators over increasing size of training data (averaged over ten runs).
    }
    \label{fig:main-pol-eval}
\end{minipage}\vspace{-0.3cm}
\;\;\;
\begin{minipage}[c]{0.52\linewidth}
 \subfloat[{\centering well-specified models}]{\includegraphics[width=0.48\textwidth]{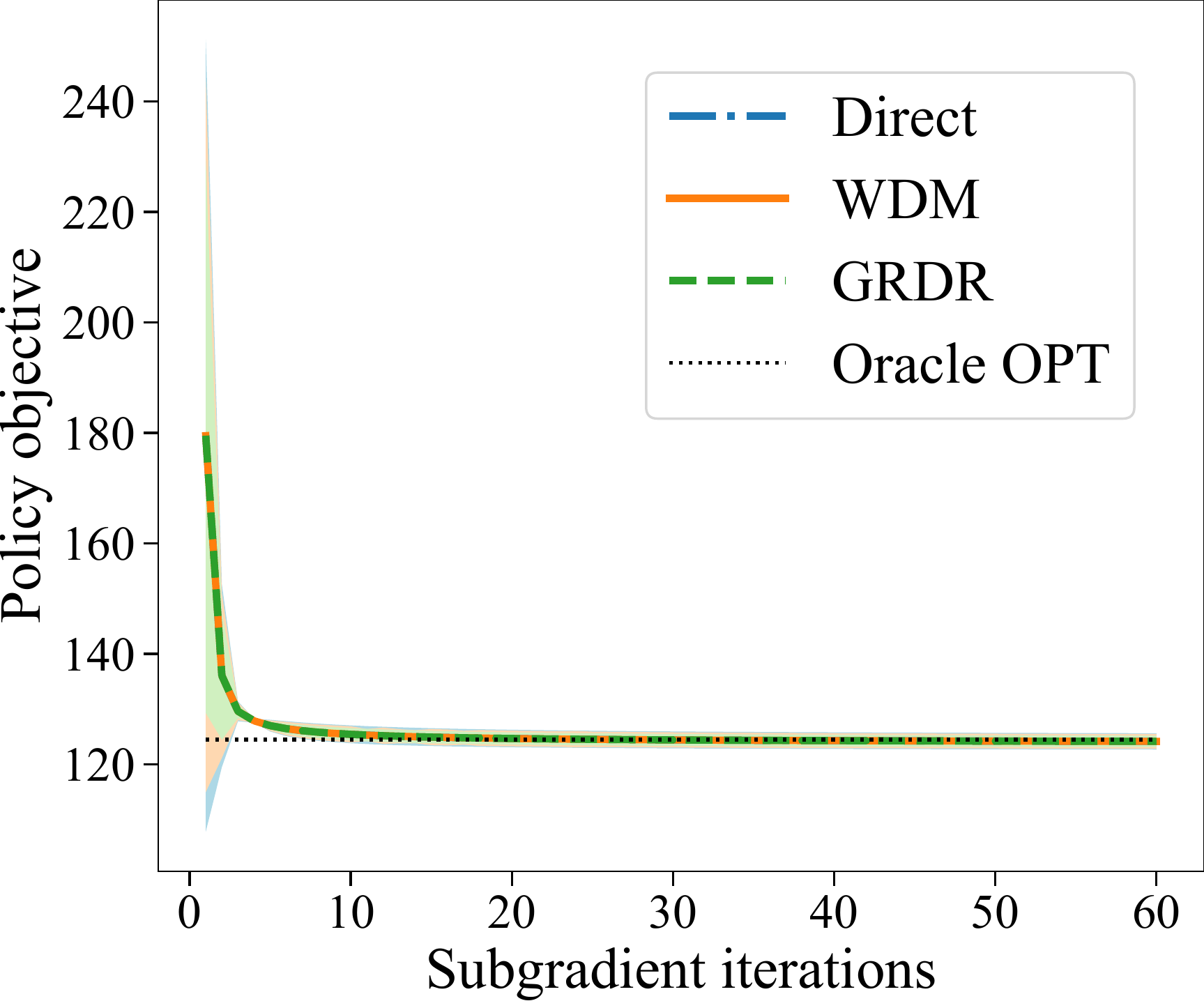}\label{fig:pol-opt-well}}\;
    \subfloat[{\centering with mis-specification}]{\includegraphics[width=0.48\textwidth]{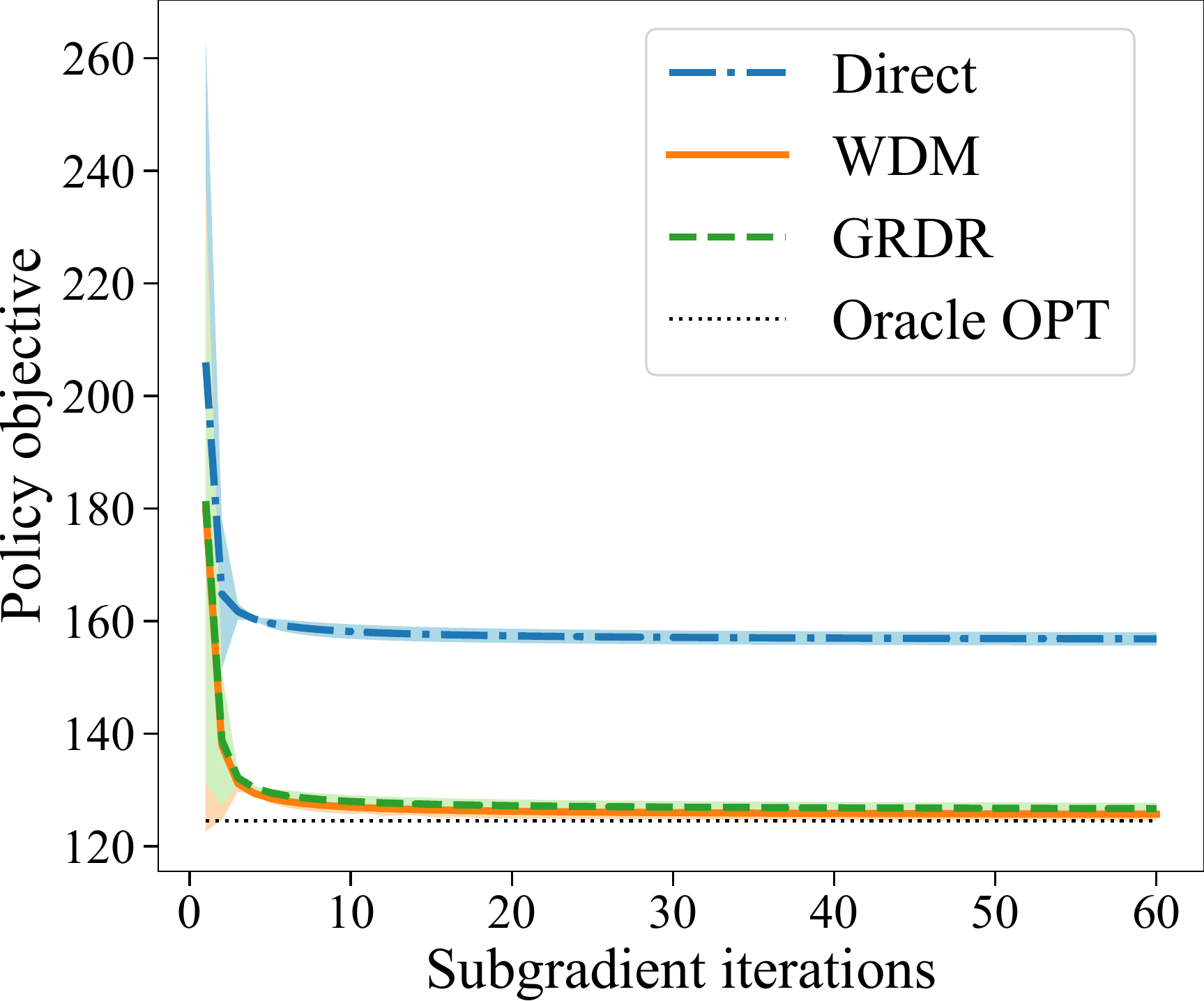}\label{fig:pol-opt-mis}}
     \caption{\emph{(Policy optimization)}.  Subgradient policy optimization with direct / $\WDM$ / $\GRDR$ estimation methods and a fixed test set. Averaged over ten random training datasets of size=1000.
    }
    \label{fig:pol-opt-fixed-test}
\end{minipage} 
\end{figure}

\looseness=-1
\textbf{3. Policy optimization.} Lastly, we integrate the policy evaluation and the sub-gradient method (Alg~\ref{alg:subgradient}) to conduct policy optimization. At each iteration of Alg~\ref{alg:subgradient}, the perturbation algorithm (Alg~\ref{alg:perturbation-method}) and one of the three different estimation methods are applied to evaluate the policy objective. We consider a logistic policy $\pi_t(W) = sigmoid(\varphi_1 \cdot W +\varphi_2)$. To study the convergence and the effectiveness of the subgradient algorithm for minimization, we fix a test set and perform subgradient descent over 60 iterations for each run. We average the policy values at each iteration over ten runs, where in each run we generate a random set of training data and a random initialization of the starting policy. 

We compare to the oracle estimator using the ground truth outcome model ($\text{Oracle OPT}$). Results are presented in Figure~\ref{fig:pol-opt-fixed-test}. Again, $\WDM$ and $\GRDR$ quickly converge to the oracle estimation, while the large bias of the direct method leads to poor policy optimization.
We further evaluate the impact of average random selected initial policies to the performance, and compared Figure~\ref{fig:pol-opt-fixed-test} with the results using a fixed initial policy. We observe that in this relatively low-dimensional example, 
the policy value converges to estimation-oracle-optimal 
after a few iterations (additional results and full training details in Appendix~\ref{app:exp}).

\section{Conclusion}

We studied a new framework for causal policy optimization with a \textit{policy-dependent} optimization response. We proposed evaluation algorithms and analysis
to address the fundamental challenge of an additional optimization bias. Simulations for both the policy evaluation and optimization algorithms demonstrate the effectiveness of this approach.
Interesting further directions include 
studying individual fairness of optimal allocations in applications such as school assignments or job matching, 
and/or computational improvements to the
policy optimization algorithms.


\section*{Acknowledgements}
We wish to acknowledge support from the Vannevar Bush Faculty Fellowship program under grant number N00014-21-1-2941. WG acknowledges support from a Google PhD fellowship. AZ acknowledges support from the Foundations of Data Science Institute. 

\clearpage
{
\bibliography{neurips_ref}

\begin{thebibliography}{59}
\providecommand{\natexlab}[1]{#1}
\providecommand{\url}[1]{\texttt{#1}}
\expandafter\ifx\csname urlstyle\endcsname\relax
  \providecommand{\doi}[1]{doi: #1}\else
  \providecommand{\doi}{doi: \begingroup \urlstyle{rm}\Url}\fi

\bibitem[Ackerberg et~al.(2012)Ackerberg, Chen, and
  Hahn]{ackerberg2012practical}
Daniel Ackerberg, Xiaohong Chen, and Jinyong Hahn.
\newblock A practical asymptotic variance estimator for two-step semiparametric
  estimators.
\newblock \emph{Review of Economics and Statistics}, 94\penalty0 (2):\penalty0
  481--498, 2012.

\bibitem[Agniel et~al.(2018)Agniel, Kohane, and Weber]{agniel2018biases}
Denis Agniel, Isaac~S Kohane, and Griffin~M Weber.
\newblock Biases in electronic health record data due to processes within the
  healthcare system: retrospective observational study.
\newblock \emph{Bmj}, 361, 2018.

\bibitem[Athey et~al.(2017)Athey, Wager, et~al.]{athey2017efficient}
Susan Athey, Stefan Wager, et~al.
\newblock Efficient policy learning.
\newblock Technical report, 2017.

\bibitem[Azevedo and Leshno(2016)]{azevedo2016supply}
Eduardo~M Azevedo and Jacob~D Leshno.
\newblock A supply and demand framework for two-sided matching markets.
\newblock \emph{Journal of Political Economy}, 124\penalty0 (5):\penalty0
  1235--1268, 2016.

\bibitem[Bang and Robins(2005)]{bang2005doubly}
Heejung Bang and James~M Robins.
\newblock Doubly robust estimation in missing data and causal inference models.
\newblock \emph{Biometrics}, 61\penalty0 (4):\penalty0 962--973, 2005.

\bibitem[Bayraksan and Morton(2006)]{bayraksan2006assessing}
G{\"u}zin Bayraksan and David~P Morton.
\newblock Assessing solution quality in stochastic programs.
\newblock \emph{Mathematical Programming}, 108\penalty0 (2):\penalty0 495--514,
  2006.

\bibitem[Bertail et~al.(2021)Bertail, Cl{\'e}men{\c{c}}on, Guyonvarch, and
  Noiry]{bertail2021learning}
Patrice Bertail, Stephan Cl{\'e}men{\c{c}}on, Yannick Guyonvarch, and Nathan
  Noiry.
\newblock Learning from biased data: A semi-parametric approach.
\newblock In \emph{International Conference on Machine Learning}, pages
  803--812. PMLR, 2021.

\bibitem[Bhattacharya(2009)]{bhattacharya2009inferring}
Debopam Bhattacharya.
\newblock Inferring optimal peer assignment from experimental data.
\newblock \emph{Journal of the American Statistical Association}, 104\penalty0
  (486):\penalty0 486--500, 2009.

\bibitem[Cao et~al.(2009)Cao, Tsiatis, and Davidian]{cao2009improving}
Weihua Cao, Anastasios~A Tsiatis, and Marie Davidian.
\newblock Improving efficiency and robustness of the doubly robust estimator
  for a population mean with incomplete data.
\newblock \emph{Biometrika}, 96\penalty0 (3):\penalty0 723--734, 2009.

\bibitem[Chandak et~al.(2021)Chandak, Niekum, da~Silva, Learned-Miller,
  Brunskill, and Thomas]{chandak2021universal}
Yash Chandak, Scott Niekum, Bruno~Castro da~Silva, Erik Learned-Miller, Emma
  Brunskill, and Philip~S Thomas.
\newblock Universal off-policy evaluation.
\newblock \emph{arXiv preprint arXiv:2104.12820}, 2021.

\bibitem[Chernozhukov et~al.(2018)Chernozhukov, Chetverikov, Demirer, Duflo,
  Hansen, Newey, and Robins]{chernozhukov2018double}
Victor Chernozhukov, Denis Chetverikov, Mert Demirer, Esther Duflo, Christian
  Hansen, Whitney Newey, and James Robins.
\newblock Double/debiased machine learning for treatment and structural
  parameters, 2018.

\bibitem[Chernozhukov et~al.(2021)Chernozhukov, Newey, Quintas-Martinez, and
  Syrgkanis]{chernozhukov2021riesznet}
Victor Chernozhukov, Whitney~K Newey, Victor Quintas-Martinez, and Vasilis
  Syrgkanis.
\newblock Riesznet and forestriesz: Automatic debiased machine learning with
  neural nets and random forests.
\newblock \emph{arXiv preprint arXiv:2110.03031}, 2021.

\bibitem[Cr{\'e}pon and Van Den~Berg(2016)]{crepon2016active}
Bruno Cr{\'e}pon and Gerard~J Van Den~Berg.
\newblock Active labor market policies.
\newblock \emph{Annual Review of Economics}, 8:\penalty0 521--546, 2016.

\bibitem[Cr{\'e}pon et~al.(2013)Cr{\'e}pon, Duflo, Gurgand, Rathelot, and
  Zamora]{crepon2013labor}
Bruno Cr{\'e}pon, Esther Duflo, Marc Gurgand, Roland Rathelot, and Philippe
  Zamora.
\newblock Do labor market policies have displacement effects? evidence from a
  clustered randomized experiment.
\newblock \emph{The Quarterly Journal of Economics}, 128\penalty0 (2):\penalty0
  531--580, 2013.

\bibitem[Danskin(1966)]{danskin1966theory}
John~M Danskin.
\newblock The theory of max-min, with applications.
\newblock \emph{SIAM Journal on Applied Mathematics}, 14\penalty0 (4):\penalty0
  641--664, 1966.

\bibitem[Dud{\'\i}k et~al.(2011)Dud{\'\i}k, Langford, and Li]{dudik2011doubly}
Miroslav Dud{\'\i}k, John Langford, and Lihong Li.
\newblock Doubly robust policy evaluation and learning.
\newblock \emph{arXiv preprint arXiv:1103.4601}, 2011.

\bibitem[Dud{\'\i}k et~al.(2014)Dud{\'\i}k, Erhan, Langford, and
  Li]{dudik2014doubly}
Miroslav Dud{\'\i}k, Dumitru Erhan, John Langford, and Lihong Li.
\newblock Doubly robust policy evaluation and optimization.
\newblock \emph{Statistical Science}, 29\penalty0 (4):\penalty0 485--511, 2014.

\bibitem[Finlayson et~al.(2021)Finlayson, Subbaswamy, Singh, Bowers, Kupke,
  Zittrain, Kohane, and Saria]{finlayson2021clinician}
Samuel~G Finlayson, Adarsh Subbaswamy, Karandeep Singh, John Bowers, Annabel
  Kupke, Jonathan Zittrain, Isaac~S Kohane, and Suchi Saria.
\newblock The clinician and dataset shift in artificial intelligence.
\newblock \emph{The New England Journal of Medicine}, 385\penalty0
  (3):\penalty0 283--286, 2021.

\bibitem[Gupta et~al.(2021)Gupta, Huang, and
  Rusmevichientong]{gupta2021debiasing}
Vishal Gupta, Michael Huang, and Paat Rusmevichientong.
\newblock Debiasing in-sample policy performance for small-data, large-scale
  optimization.
\newblock \emph{Large-Scale Optimization (June 2, 2021)}, 2021.

\bibitem[Hardt et~al.(2016)Hardt, Megiddo, Papadimitriou, and
  Wootters]{hardt2016strategic}
Moritz Hardt, Nimrod Megiddo, Christos Papadimitriou, and Mary Wootters.
\newblock Strategic classification.
\newblock In \emph{Proceedings of the 2016 ACM Conference on Innovations in
  Theoretical Computer Science}, pages 111--122, 2016.

\bibitem[Hirano and Porter(2020)]{hirano2020asymptotic}
Keisuke Hirano and Jack~R Porter.
\newblock Asymptotic analysis of statistical decision rules in econometrics.
\newblock In \emph{Handbook of Econometrics}, volume~7, pages 283--354.
  Elsevier, 2020.

\bibitem[Hudgens and Halloran(2008)]{hudgens2008toward}
Michael~G Hudgens and M~Elizabeth Halloran.
\newblock Toward causal inference with interference.
\newblock \emph{Journal of the American Statistical Association}, 103\penalty0
  (482):\penalty0 832--842, 2008.

\bibitem[Ito et~al.(2018)Ito, Yabe, and Fujimaki]{ito2018unbiased}
Shinji Ito, Akihiro Yabe, and Ryohei Fujimaki.
\newblock Unbiased objective estimation in predictive optimization.
\newblock In \emph{International Conference on Machine Learning}, pages
  2176--2185. PMLR, 2018.

\bibitem[Kallus and Zhou(2018{\natexlab{a}})]{kallus2018confounding}
Nathan Kallus and Angela Zhou.
\newblock Confounding-robust policy improvement.
\newblock \emph{Advances in Neural Information Processing Systems}, 31,
  2018{\natexlab{a}}.

\bibitem[Kallus and Zhou(2018{\natexlab{b}})]{kallus2018policy}
Nathan Kallus and Angela Zhou.
\newblock Policy evaluation and optimization with continuous treatments.
\newblock In \emph{International Conference on Artificial Intelligence and
  Statistics}, pages 1243--1251. PMLR, 2018{\natexlab{b}}.

\bibitem[Kallus and Zhou(2021)]{kallus2021minimax}
Nathan Kallus and Angela Zhou.
\newblock Minimax-optimal policy learning under unobserved confounding.
\newblock \emph{Management Science}, 67\penalty0 (5):\penalty0 2870--2890,
  2021.

\bibitem[Kannan et~al.(2020)Kannan, Bayraksan, and Luedtke]{kannan2020data}
Rohit Kannan, G{\"u}zin Bayraksan, and James~R Luedtke.
\newblock Data-driven sample average approximation with covariate information.
\newblock \emph{Optimization Online. URL: http://www. optimization-online.
  org/DB\_HTML/2020/07/7932. html}, 2020.

\bibitem[Kennedy(2016)]{kennedy2016semiparametric}
Edward~H Kennedy.
\newblock Semiparametric theory and empirical processes in causal inference.
\newblock In \emph{Statistical Causal Inferences and their Applications in
  Public Health Research}, pages 141--167. Springer, 2016.

\bibitem[Kitagawa and Tetenov(2018)]{kitagawa2018should}
Toru Kitagawa and Aleksey Tetenov.
\newblock Who should be treated? empirical welfare maximization methods for
  treatment choice.
\newblock \emph{Econometrica}, 86\penalty0 (2):\penalty0 591--616, 2018.

\bibitem[Knaus et~al.(2020)Knaus, Lechner, and
  Strittmatter]{knaus2020heterogeneous}
Michael~C Knaus, Michael Lechner, and Anthony Strittmatter.
\newblock Heterogeneous employment effects of job search programmes: A machine
  learning approach.
\newblock \emph{Journal of Human Resources}, pages 0718--9615R1, 2020.

\bibitem[Kube et~al.(2019)Kube, Das, and Fowler]{kube2019allocating}
Amanda Kube, Sanmay Das, and Patrick~J Fowler.
\newblock Allocating interventions based on predicted outcomes: A case study on
  homelessness services.
\newblock In \emph{Proceedings of the AAAI Conference on Artificial
  Intelligence}, volume~33, pages 622--629, 2019.

\bibitem[K{\"u}nzel et~al.(2019)K{\"u}nzel, Sekhon, Bickel, and
  Yu]{kunzel2019metalearners}
S{\"o}ren~R K{\"u}nzel, Jasjeet~S Sekhon, Peter~J Bickel, and Bin Yu.
\newblock Metalearners for estimating heterogeneous treatment effects using
  machine learning.
\newblock \emph{Proceedings of the National Academy of Sciences}, 116\penalty0
  (10):\penalty0 4156--4165, 2019.

\bibitem[Lam and Qian(2018)]{lam2018bounding}
Henry Lam and Huajie Qian.
\newblock Bounding optimality gap in stochastic optimization via bagging:
  Statistical efficiency and stability.
\newblock \emph{arXiv preprint arXiv:1810.02905}, 2018.

\bibitem[Leqi and Kennedy(2021)]{leqi2021median}
Liu Leqi and Edward~H Kennedy.
\newblock Median optimal treatment regimes.
\newblock \emph{arXiv preprint arXiv:2103.01802}, 2021.

\bibitem[Li et~al.(2021)Li, Zhao, Johari, and Weintraub]{li2021interference}
Hannah Li, Geng Zhao, Ramesh Johari, and Gabriel~Y Weintraub.
\newblock Interference, bias, and variance in two-sided marketplace
  experimentation: Guidance for platforms.
\newblock \emph{arXiv preprint arXiv:2104.12222}, 2021.

\bibitem[Lopez et~al.(2020)Lopez, Li, Yan, Xiong, Jordan, Qi, and
  Song]{lopez2020cost}
Romain Lopez, Chenchen Li, Xiang Yan, Junwu Xiong, Michael Jordan, Yuan Qi, and
  Le~Song.
\newblock Cost-effective incentive allocation via structured counterfactual
  inference.
\newblock In \emph{Proceedings of the AAAI Conference on Artificial
  Intelligence}, volume~34, pages 4997--5004, 2020.

\bibitem[Ma et~al.(2021)Ma, Fang, and Parkes]{ma2021spatio}
Hongyao Ma, Fei Fang, and David~C Parkes.
\newblock Spatio-temporal pricing for ridesharing platforms.
\newblock \emph{Operations Research}, 2021.

\bibitem[Manski(2004)]{manski2004statistical}
Charles~F Manski.
\newblock Statistical treatment rules for heterogeneous populations.
\newblock \emph{Econometrica}, 72\penalty0 (4):\penalty0 1221--1246, 2004.

\bibitem[Mejia and Parker(2021)]{mejia2021transparency}
Jorge Mejia and Chris Parker.
\newblock When transparency fails: Bias and financial incentives in ridesharing
  platforms.
\newblock \emph{Management Science}, 67\penalty0 (1):\penalty0 166--184, 2021.

\bibitem[Munro et~al.(2021)Munro, Wager, and Xu]{munro2021treatment}
Evan Munro, Stefan Wager, and Kuang Xu.
\newblock Treatment effects in market equilibrium.
\newblock \emph{arXiv preprint arXiv:2109.11647}, 2021.

\bibitem[Newey and McFadden(1994)]{newey1994large}
Whitney~K Newey and Daniel McFadden.
\newblock Large sample estimation and hypothesis testing.
\newblock \emph{Handbook of Econometrics}, 4:\penalty0 2111--2245, 1994.

\bibitem[Paxton et~al.(2013)Paxton, Niculescu-Mizil, and
  Saria]{paxton2013developing}
Chris Paxton, Alexandru Niculescu-Mizil, and Suchi Saria.
\newblock Developing predictive models using electronic medical records:
  challenges and pitfalls.
\newblock In \emph{AMIA Annual Symposium Proceedings}, volume 2013, page 1109.
  American Medical Informatics Association, 2013.

\bibitem[Perdomo et~al.(2020)Perdomo, Zrnic, Mendler-D{\"u}nner, and
  Hardt]{perdomo2020performative}
Juan Perdomo, Tijana Zrnic, Celestine Mendler-D{\"u}nner, and Moritz Hardt.
\newblock Performative prediction.
\newblock In \emph{International Conference on Machine Learning}, pages
  7599--7609. PMLR, 2020.

\bibitem[Qi et~al.(2019)Qi, Pang, and Liu]{qi2019estimating}
Zhengling Qi, Jong-Shi Pang, and Yufeng Liu.
\newblock Estimating individualized decision rules with tail controls.
\newblock \emph{arXiv preprint arXiv:1903.04367}, 2019.

\bibitem[Rahmattalabi et~al.(2022)Rahmattalabi, Vayanos, Dullerud, and
  Rice]{rahmattalabi2022learning}
Aida Rahmattalabi, Phebe Vayanos, Kathryn Dullerud, and Eric Rice.
\newblock Learning resource allocation policies from observational data with an
  application to homeless services delivery.
\newblock \emph{arXiv preprint arXiv:2201.10053}, 2022.

\bibitem[Scharfstein et~al.(1999)Scharfstein, Rotnitzky, and
  Robins]{scharfstein1999adjustingrejoinder}
Daniel~O Scharfstein, Andrea Rotnitzky, and James~M Robins.
\newblock Adjusting for nonignorable drop-out using semiparametric nonresponse
  models (rejoinder).
\newblock \emph{Journal of the American Statistical Association}, 94\penalty0
  (448):\penalty0 1135--1146, 1999.

\bibitem[Shalit et~al.(2017)Shalit, Johansson, and
  Sontag]{shalit2017estimating}
Uri Shalit, Fredrik~D Johansson, and David Sontag.
\newblock Estimating individual treatment effect: generalization bounds and
  algorithms.
\newblock In \emph{International Conference on Machine Learning}, pages
  3076--3085. PMLR, 2017.

\bibitem[Shapiro et~al.(2021)Shapiro, Dentcheva, and
  Ruszczynski]{shapiro2021lectures}
Alexander Shapiro, Darinka Dentcheva, and Andrzej Ruszczynski.
\newblock \emph{Lectures on Stochastic Programming: Modeling and Theory}.
\newblock SIAM, 2021.

\bibitem[Shi et~al.(2019)Shi, Blei, and Veitch]{shi2019adapting}
Claudia Shi, David~M Blei, and Victor Veitch.
\newblock Adapting neural networks for the estimation of treatment effects.
\newblock \emph{arXiv preprint arXiv:1906.02120}, 2019.

\bibitem[Shimodaira(2000)]{shimodaira2000improving}
Hidetoshi Shimodaira.
\newblock Improving predictive inference under covariate shift by weighting the
  log-likelihood function.
\newblock \emph{Journal of Statistical Planning and Inference}, 90\penalty0
  (2):\penalty0 227--244, 2000.

\bibitem[Sun(2021)]{sun2021empirical}
Liyang Sun.
\newblock Empirical welfare maximization with constraints.
\newblock \emph{arXiv preprint arXiv:2103.15298}, 2021.

\bibitem[Swaminathan and Joachims(2015)]{swaminathan2015batch}
Adith Swaminathan and Thorsten Joachims.
\newblock Batch learning from logged bandit feedback through counterfactual
  risk minimization.
\newblock \emph{The Journal of Machine Learning Research}, 16\penalty0
  (1):\penalty0 1731--1755, 2015.

\bibitem[Thomas et~al.(2015)Thomas, Theocharous, and
  Ghavamzadeh]{thomas2015high}
Philip Thomas, Georgios Theocharous, and Mohammad Ghavamzadeh.
\newblock High-confidence off-policy evaluation.
\newblock In \emph{Proceedings of the AAAI Conference on Artificial
  Intelligence}, volume~29, 2015.

\bibitem[Tran et~al.(2019)Tran, Yiannoutsos, Wools-Kaloustian, Siika, Van
  Der~Laan, and Petersen]{tran2019double}
Linh Tran, Constantin Yiannoutsos, Kara Wools-Kaloustian, Abraham Siika, Mark
  Van Der~Laan, and Maya Petersen.
\newblock Double robust efficient estimators of longitudinal treatment effects:
  Comparative performance in simulations and a case study.
\newblock \emph{The international Journal of Biostatistics}, 15\penalty0 (2),
  2019.

\bibitem[van~de Geer and Stougie(1991)]{van1991rates}
Sara van~de Geer and Leen Stougie.
\newblock On rates of convergence and asymptotic normality in the multiknapsack
  problem.
\newblock \emph{Mathematical Programming}, 51\penalty0 (1):\penalty0 349--358,
  1991.

\bibitem[Wager and Athey(2018)]{wager2018estimation}
Stefan Wager and Susan Athey.
\newblock Estimation and inference of heterogeneous treatment effects using
  random forests.
\newblock \emph{Journal of the American Statistical Association}, 113\penalty0
  (523):\penalty0 1228--1242, 2018.

\bibitem[Wager and Xu(2021)]{wager2021experimenting}
Stefan Wager and Kuang Xu.
\newblock Experimenting in equilibrium.
\newblock \emph{Management Science}, 2021.

\bibitem[Wang et~al.(2019)Wang, Bai, Bhalla, and Joachims]{wang2019batch}
Lequn Wang, Yiwei Bai, Arjun Bhalla, and Thorsten Joachims.
\newblock Batch learning from bandit feedback through bias corrected reward
  imputation.
\newblock In \emph{ICML Workshop on Real-World Sequential Decision Making},
  2019.

\bibitem[Zhao et~al.(2012)Zhao, Zeng, Rush, and Kosorok]{zhao2012estimating}
Yingqi Zhao, Donglin Zeng, A~John Rush, and Michael~R Kosorok.
\newblock Estimating individualized treatment rules using outcome weighted
  learning.
\newblock \emph{Journal of the American Statistical Association}, 107\penalty0
  (499):\penalty0 1106--1118, 2012.

\end{thebibliography}
\bibliographystyle{plainnat}
}
\clearpage
\newpage

\newpage
\appendix



\vspace{5mm}
\begin{center}

  {\bf{\LARGE{Appendix}}}
\end{center}

\vspace{3mm}

\begin{table}[H]
\begin{tabular}{ll}

Appendix Section                                                                                                           & Referencing Section                                                                                                 \\ \midrule
\Cref{apx-relatedwork}
& \Cref{sec:related_work}, Related Work                                                                               \\
\Cref{apx-estimation}
& \begin{tabular}[c]{@{}l@{}}\Cref{sec:estimation}, Causal Estimation with Policy-Dependent Responses\end{tabular} \\
\Cref{apx-alt-asymptotic-regime}
& \Cref{sec:prelim,sec:estimation}, Preliminaries / Causal Estimation with Policy-Dependent Responses                   \\
\Cref{sec:decdependentclassifier}
&                \Cref{sec:prelim}, Preliminaries                                                                                                     \\
\Cref{app:exp}
& \Cref{sec:experiments}, Experimental Evaluation                                                       
\end{tabular}
\caption{Table of contents: Appendix and referring sections.} 
\end{table}

\vspace{5mm}

\section{Further Related Work and Comparisons}\label{apx-relatedwork}

\paragraph{Other variants of off-policy evaluation and optimization with global dependence on the population. }
There is also a line of work on off-policy evaluation, for example evaluating policies by non-average functionals over populations (median \citep{leqi2021median}, quantile \citep{chandak2021universal,qi2019estimating}); but estimation crucially depends on specific reformulations based on problem structure. These functionals are typically risk deviations and reformulated in relation to estimation of quantiles rather than generic optimization formulations. In contrast, our debiasing is more general and independent of the functional form of the risk deviation beyond the second-stage problem being a stochastic linear optimization problem. However, risk deviations typically lead to \textit{min-max} problems, while our formulation has aligned objective functions of treatment and response and is inherently a partial optimization of a nonconvex problem, hence a \textit{min-min} problem. 

Such min-max formulations also appear in sensitivity analysis and approaches to unobserved confounding via min-max robustness \citep{kallus2018confounding,kallus2021minimax} and similar algorithms for policy optimization appear there. However, the formulation of this work focuses specifically on \textit{generic optimization problems} in a different conceptual setting. Estimation is quite different due to the requirement of compatibility of estimation and a perturbation approach. 

Our focus on the downstream global system response bears a distant conceptual resemblance to interference \citep{hudgens2008toward}, but is fundamentally very different: we require the \textit{causal} response satisfies the stable unit treatment value assumption (SUTVA), while the source of interaction across units (analogous to the exposure mapping) \textit{is completely known to} and \textit{under the control of} the decision-maker. 

Lastly, recent work develops specialized estimation more closely using the structure of economic systems, such as marketplace interference \citep{li2021interference} or mean-field equilibrium \citep{wager2021experimenting,munro2021treatment}. This work is often closely coupled with the application structure. Our use of policy-dependent structure is coarse, considering only a generic linear optimization response and in turn adjusting for the introduced optimization bias.

\paragraph{Decision-dependent classifier shift.} Our last example of estimating decision-dependent predictive loss is broadly motivated by decision-dependent distribution shifts but focuses on settings with distinct treatments, rather than other frameworks where a classifier is a treatment studying model-based or utility-model based approaches \citep{hardt2016strategic,perdomo2020performative}. 

\paragraph{Structured prediction.} Finally, although there is extensive work on structured prediction in machine learning, our framework is very different: while structured prediction maps contextual inputs to the space of complex outputs (the space of network flows, matchings; the optimization decision vector), our data environment consists of contexts and separable, unit-dependent outcomes.  
 
 \clearpage
\section{Additional Details for Estimation}\label{apx-estimation}

\subsection{Preliminaries}\label{apx-estimation-preliminaries} 

\paragraph{Adjusting for confounding: IPW and AIPW.}
In general, plug-in estimation of $\estmu_t(W)$ does \textit{not} admit unbiased predictions because of selection bias and model misspecification. A key object that adjusts for selection bias is the  

{\centering
  $ \displaystyle
\begin{aligned}
  \textit{propensity score: } e_t(W) = \prob(T=t \mid W).
\end{aligned}
  $ 
\par}
Although importance sampling cannot \textit{directly} be applied in our main regime of interest with out-of-sample evaluation as in \Cref{asn-asymptotic regime}, we introduce key properties which can be used to debias outcome models. (See \Cref{apx-alt-asymptotic-regime} for discussion of an alternative in-sample OPE regime). 

\textit{Inverse propensity weighting (IPW)} transforms treatment-conditional expectations to the population expectation---by
 iterated expectations and \Cref{asn-ignorability} (ignorability), we have:
\begin{equation}\label{eqn-iteratedexpectation}
\sum_{t} \E \left[\E \left[ { c}\frac{\mathbb{I}[Z_\pi=t]}{ e_t(W)}\mid W \right]  \right] =\sum_{t} \E \left[ { c}\frac{ \mathbb{I}[Z_\pi=t]}{e_t(W)} \right]  = \sum_{t}  \E [ c(\pi_t) ] = \E[ c(\pi) ]. 
\end{equation}
In general, \textit{doubly-robust augmented inverse probability weighting (AIPW)} estimation improves the variance when both models are well-specified and achieves overall unbiasedness under unbiasedness of either outcome or propensity:
$$\textstyle \sum_{t}\E \left[ 
\pi_t(W)\mathbb{I}[Z_\pi=t]( \frac{\mathbb{I}[T=t]}{e_t(W)}(c-\mu_t(W))  + \mu_t(W) ) \right] =\textstyle\E[ c(\pi) ].$$
 \section{Proofs}
 
\paragraph{Asymptotic variance of two-step estimation via GMM asymptotic variance} 
We first recall a general framework for deriving asymptotic variance with generated regressors as discussed in \cite{newey1994large}. This section is summarized from the presentation in there to keep derivations self-contained. 

We focus on the approach based on the asymptotic variance for GMM, viewing the nuisance estimations as ``stacked" moment equations for $\hat \theta$ (second-stage estimate) and $\hat \gamma$ (the first-stage estimation). Then, applying blockwise inversion to the GMM asymptotic variance obtains the asymptotic variance of the (parameter) vector, $\begin{bmatrix}
\theta \\ \gamma
\end{bmatrix} $ in terms of the first component (top left block). 

Stack the moment equations for $\hat \theta, \hat \gamma$ to get: 
\begin{align*}
    \mathbb E [ g(W, \theta_0, \gamma_0)  ] &= 0, \\
    \mathbb E [ h(W, \gamma_0 ] &= 0. 
\end{align*}
Note adaptivity occurs iff 
$G_\gamma = 
\nabla_\gamma \E[ g(W, \theta_0, \gamma_0) ] 
=0
$




Define the following sample Jacobians of moment conditions with respect to the parameters: 
\begin{align*}
    \hat G_\theta &= \avgn \nabla_\theta g(w_i, \hat \theta, \hat \gamma),\qquad\hat G_\gamma = \avgn \nabla_\gamma g(w_i, \hat \theta, \hat \gamma), \qquad 
    \hat H = \avgn \nabla_\gamma h(w_i, \hat \gamma).
\end{align*}

For short introduce notation for the empirical moments $\hat g_i, \hat h_i$:
$$
\hat g_i = g(w_i, \hat\theta, \hat \gamma), 
\qquad 
\hat h_i = h(w_i, \hat \gamma), 
$$
so that we can define the sample second moment matrix $\hat\Omega$ as follows: 
$$
\hat \Omega = \avgn (\hat g_i, \hat h_i)^\top 
(\hat g_i, \hat h_i).
$$

The estimator for asymptotic variance is given by: 

\begin{align*}\hat V 
&= 
\begin{bmatrix} 
\hat{G}_\theta & \hat{G}_\gamma \\ 
0 & \hat H 
\end{bmatrix}^{-1}
\Omega^{-1}
\begin{bmatrix} 
\hat{G}_\theta & \hat{G}_\gamma \\ 
0 & \hat H
\end{bmatrix}^{-1} \\
&=\left[\begin{array}{cc}
\hat{G}_{\theta}^{-1} & -\widehat{G}_{\theta}^{-1} \hat{G}_{\gamma} \hat{H}^{-1} \\
0 & \hat{H}^{-1}
\end{array}\right] \widehat{\Omega}\left[\begin{array}{cc}
\hat{G}_{\theta}^{-1} & -\widehat{G}_{\theta}^{-1} \widehat{G}_{\gamma} \hat{H}^{-1} \\
0 & \hat{H}^{-1}
\end{array}\right].
\end{align*}

If the moment functions are uncorrelated then the first-step estimation increases the second-step variance. We may obtain a simplification as follows. Let $\hat\phi_i = - \hat{H}^{-1} \hat h_i$. Then the upper left block is 
$$\hat{V}_{\theta}=\widehat{G}_{\theta}^{-1}\left[n^{-1} \sum_{i=1}^{n}\left\{\hat{g}_{i}+\hat{G}_{\gamma} \hat{\psi}_{i}\right\}\left\{\hat{g}_{i}+\widehat{G}_{\gamma} \hat{\psi}_{i}\right\}^{\prime}\right] \widehat{G}_{\theta}^{-1}.$$ 
For $\hat V_\gamma = 
\avgn \hat\phi_i \hat\phi_i ',$ an asymptotic variance estimator for $\hat \theta$ is 
$$
\hat V_\theta = 
\hat G_\theta^{-1}
\left( \avgn \hat g_i  \hat g_i'
\right)
(\hat G_\theta^{-1})'
+ \hat G_\theta^{-1} \hat G_\gamma \hat V_\gamma (\hat G_\gamma)' (\hat{G}_\theta^{-1})'.
$$

\subsubsection{Outcome regressions with generated regressors}

In the appendix we also discuss an approach based on $\GRDR.$ 
\begin{proposition}[Asymptotic normality of $\GRDR$]\label{prop-avar-grdr}
Let $e_t(w), \mu(w)$ satisfy \cref{eqn-estasn-momentcondition} with the moment condition for $\mu$ given by \Cref{eqn-esteqn-grdr}. 
Then 
$$\hat{G}_\gamma
= \begin{bmatrix}
\E_n[2(\epsilon_1 (\frac{\partial }{\partial \gamma} e_1^{-1}(W;\gamma)) ) +\epsilon_0 (\frac{\partial }{\partial \gamma} e_0^{-1}(W;\gamma)) )\frac{\partial \mu}{\partial \tilde\theta}\\
- \E_n[2T(c- \mu(W; \theta))  (\frac{\partial }{\partial \gamma} e^{-1}_1(W;\gamma))]\\
- \E_n[2 (1-T) (c- \mu(W; \theta)) (\frac{\partial }{\partial \gamma} e^{-1}_0(W;\gamma))]
\end{bmatrix}.$$
\end{proposition}
\begin{proof}[Proof of \Cref{prop-avar-weighteddm}]

$$ \mathbb{E}_n[ \nabla_\gamma g(W; \theta,\beta) ]= \mathbb{E}_n
\left[ T (\frac{\partial e}{\partial \gamma} e^{-1} )\cdot 2 (c - \mu(W;\beta)) \frac{\partial \mu}{\partial \tilde\theta} \right].  $$

\end{proof}

\begin{proof}[Proof of \Cref{prop-avar-grdr}, $\GRDR$, \cite{bang2005doubly}]

The stacked estimation equations are as follows, for the parameters $
[\gamma, \theta, \epsilon_1, \epsilon_0]
$: 

\begin{align*}
&- \E\left[2(T-e_T(W;\gamma)) \frac{\partial e}{\partial \gamma} \right] = 0, \\
& - \E\left[2(c-\tilde{\mu}(W;\theta) \frac{\partial \mu}{\partial \theta} \right] = 0, \\
& - \E[2T(c-\tilde{\mu}_1(W;\theta)  e_1^{-1}(W;\gamma) ] = 0, \\
& - \E[2(1-T)(c-\tilde{\mu}_0(W;\theta)  e_0^{-1}(W;\gamma) ] = 0 
\end{align*}
and the Jacobian of partial derivatives is: 
$$ G_\gamma
= \begin{bmatrix}
\E[2(\epsilon_1 (\frac{\partial }{\partial \gamma} e_1^{-1}(W;\gamma)) ) +\epsilon_0 (\frac{\partial }{\partial \gamma} e_0^{-1}(W;\gamma)) )\frac{\partial \mu}{\partial \theta}\\
- \E[2 (c- \mu(W; \theta))T (\frac{\partial }{\partial \gamma} e^{-1}_1(W;\gamma))]\\
- \E[2 (c- \mu(W; \theta)) (1-T) (\frac{\partial }{\partial \gamma} e^{-1}_0(W;\gamma)]
\end{bmatrix}.$$
\end{proof}

\subsection{Nonlinear Generalization of Perturbation Method}

\paragraph{Preliminaries.}
We include the proof for completeness. It is the same argument of \citet{ito2018unbiased} with the addition of linear expansions of the nonlinear model $\gfn$ around the parameter $\perturbparam$.

Let $\gfn$ denote a generic prediction model which may depend nonlinearly upon its parameter $\perturbparam$. Define for our context the true-optimal-decision $\z(\perturbparam^\ast)$, and sample-optimal-decision $\hat\z(\hth)$: 
\begin{align*}
    \z^* &\in \argmax \sum_{i=1}^m \mu^*( W_i,\perturbparam^*) \z_i ,\\
    \hat{\z} &\in \argmax \sum_{i=1}^m \gxt{\hth} \z_i.
\end{align*}
Define the auxiliary functions $\eta, \phi$ evaluated along paths indexed by $\epsilon$: 
\begin{align*}
    \eta(\pathg) &= \mathbb{E}_\delta \left[ \sum_{i=1}^m \z(\perturbparam^*+\pathg \delta) \gxt{\tht^*} \right],\\
    \phi(\pathg) &= \mathbb{E}_\delta \left[ \sum_{i=1}^m \z(\perturbparam^*+\pathg \delta) \gxt{\tht^*+\pathg \delta} \right].
\end{align*}
 
We focus on the case exclusively where the function $f$ of interest is affine, i.e. so that $f(z^*, \perturbparam^*) = \sum_{i=1}^m z_i^* g(\perturbparam^*; X_i)$. 
 

\begin{proof}[Proof of \Cref{prop-pathderiv}]
Let $\thtp = \tht^*+\pathg\delta$.
First, we will show 
\begin{equation}
    \eta(\pathg)-\phi(\pathg) = \pathg
    \; 
    \mathbb{E}_\delta\left[ \avgm \hat{\z}_i (\nabla_\perturbparam \gfn \Big\vert_{\tht_\pathg}  \delta ) \right] + O(\epsilon^2). \label{eqn-perturbationproof1} 
\end{equation}
and then that $\pathg \phi'(\pathg)$ equals the right-hand-side of the above. 

\underline{Step 1} (Showing \Cref{eqn-perturbationproof1}): 

Expand the definition of $\eta, \phi$ and apply a Taylor expansion of $\gxt{\tht^*+\pathg\delta}$ from $\gxt{\tht^*}$. 
\begin{align*}
       \eta(\pathg)-\phi(\pathg) &= \mathbb{E}_\delta \left[ \avgm \left( \hat{\z}_i \gxt{\tht^*} - \hat{\z}_i \gxt{\hth}\right) \right]= \mathbb{E}_\delta \left[ \avgm\hat{\z}_i \left(  \gxt{\tht^*} -  \gxt{\hth}\right) \right]\\
        &= \mathbb{E}_\delta \left[\avgm \hat{\z}_i \left( \pathg (\nabla_\perturbparam \gfn \Big\vert_{\tht_\pathg}  \delta ) + O( \|\pathg\delta\|^2_2) \right) \right].
\end{align*}

\underline{Step 2}: $\pathg\phi'(\pathg) = \text{RHS of}$ \Cref{eqn-perturbationproof1}. 

Let $\thtph = \tht^*+(\pathg+h)\delta$.

By definition, 
\begin{align*}
    \phi'(\pathg) &= \lim_{h\to 0} \frac{\phi(\pathg+h) - \phi(\pathg) }{h}\\
    &=\lim_{h\to 0} \frac 1h \left(
    \mathbb{E}_\delta \left[ \avgm
    {\z}_i(\thtph) \gxt{\thtph}- 
    \avgm
    {\z}_i(\thtp) \gxt{\thtp}
    \right]
    \right) .
\end{align*}
Add / subtract $\hat{\z}(\tht_\pathg)\gxt{\tht_\pathg}:$
\begin{align*}
    \phi'(\pathg) &=\lim_{h\to 0} 
    \left\{ \frac 1h 
    \left( 
    \mathbb{E}_\delta \left[ \avgm\left(
    {\z}_i(\thtph) \gxt{\thtp}
    +
    \z_i(\thtph)(\gxt{\thtph}-\gxt{\thtp}) \right)
    \right]\right) 
    - \frac{1}{h} \left(
    \mathbb{E}_\delta \left[ \avgm {\z}_i(\thtp)\gxt{\thtp} \right) \right]
    \right\}\\
    &= \lim_{h\to 0} 
    \left\{
    \frac 1h \left( \mathbb{E}_\delta \left[ 
  \avgm  ( \z_i(\thtph)- \z_i(\thtp)) \gxt{\thtp}
  +  \avgm  ( \z_i(\thtph)(\gxt{\thtph}- \gxt{\thtp})
  \right]
    \right) 
    \right\}\\
    &= \lim_{h\to 0} 
    \left\{
    \frac 1h \left( \mathbb{E}_\delta \left[ 
  {\avgm  ( \z_i(\thtph)- \z_i(\thtp)) \gxt{\thtp}}
  +  \avgm  ( \z_i(\thtph)(\gxt{\thtph}- \gxt{\thtp})
  \right]
    \right) 
    \right\}.
\end{align*}
The last line follows by a Taylor expansion of $\gfn$ from $\thtp$ to $\thtph$ and noting that the first term converges as $\z$ does, $
\textstyle \lim_{h\to 0} 
    \frac 1h ( 
  {\avgm  ( \z_i(\thtph)- \z_i(\thtp)) \gxt{\thtp}}) = 0.$
  under regularity conditions common in perturbation analysis of stochastic programs, such as uniqueness of the solution. 
  
  Therefore, interchanging limits and the expectation: 
  \begin{align*}
        \phi'(\pathg)  &= \mathbb{E}_\delta \left[ \lim_{h\to 0} 
    \left\{
    \frac 1h \left( 
 \avgm   \z_i(\thtph)(\nabla_\perturbparam \gfn \Big\vert_{\tht_\pathg}  (h\delta)
 + O(\norm{h \pathg}_2^2 ) 
 )
    \right) 
    \right\}
    \right]\\
    &=\mathbb{E}_\delta \left[  \lim_{h\to 0} 
    \left\{ \avgm   \z_i(\thtph)(\nabla_\perturbparam \gfn \Big\vert_{\tht_\pathg}  \delta
 + O(h)\right\} \right]\\
 &= \mathbb{E}_\delta \left[ \avgm   \z_i(\thtp)(\nabla_\perturbparam \gfn \Big\vert_{\tht_\pathg}\delta)\right].
  \end{align*}
\end{proof}

\clearpage
\section{Alternative Asymptotic Regime: In-sample, growing-dimension (\Cref{asn-insample-regime})}\label{apx-alt-asymptotic-regime}  
In the main text we focused on a fixed-dimension regime. We describe some extensions that may be possible to handle an in-sample, growing dimension, growing-n regime described in \Cref{asn-insample-regime}. We do generally require additional structural information to apply more familiar OPE estimators such as IPW/AIPW and other adaptations of bias adjustment methods. 

The strongest such additional structural knowledge is that the optimization is highly structured so as to admit a finite VC dimension; to circumvent issues related to the growing dimension. 
\begin{assumption}\label{asn-vcx} 
$x(\pi,W)$ has finite VC dimension.
\end{assumption}
Such a structural characterization is established in special cases such as multi-knapsack linear programs \citep{van1991rates}, or large-market limits of stable matching markets \citep{azevedo2016supply}; but need not hold in general. 

\subsection{Preliminaries}
For completeness, we describe the analogous estimands/estimators for the \textit{in-sample, growing-n} regime as described in the main text for the out-of-sample, fixed-m regime. 

Plug-in estimation is evaluated on the training dataset as follows: 
\begin{equation}
  \estmuv_{\pi} = \underset{x \in \cX}{\min}
  \left\{ \avgn \sumt   \pi_t \estmu_t(w_i) x_i \colon {Ax \leq b} \right\}.
\end{equation}

We describe extensions of approaches to handle in-sample bias in this growing-dimension setting, although we generally require more structure on the problem. As described in \Cref{asn-insample-regime} we typically require a problem-dependent asymptotic scaling; for example that we jointly scale up the problem size as well as the constraints. We provide a concrete example for \Cref{example:matching-part-1}.  
\begin{example}[Fluid limit for \Cref{example:matching-part-1}]
The number of workers is $\alpha n$ and the number of jobs is $\beta n$. 
\end{example}

\subsection{Sample splitting}
We first discuss an analogous sample splitting extension of \citet{ito2018unbiased} which combines their sample splitting procedure with standard cross-fitting for doubly robust estimators \citep{chernozhukov2018double}. However, a naive extension requires four folds and is therefore expected to perform poorly in finite samples.

Let $K$ denote the number of folds, we will exposit the case of two folds and the $K$-fold generalization is standard. Denote two main folds of the data, $\mathcal{I}_{k_1}, \mathcal{I}_{k_2}$ denoting the index sets for the data, and $\mathcal{I}_{k_{1e}}, \mathcal{I}_{k_{1\mu}}$ be subfolds of $\mathcal{I}_{k_1}$ (respectively subfolds for $\mathcal{I}_{k_{2}}$). As suggested by the notation, we use distinct subfolds $\mathcal{I}_{k_{1e}}, \mathcal{I}_{k_{1\mu}}$ to learn the nuisance estimates $e, \mu$ (from the respective subfold). 

The main difference from standard cross-fitting is that in \Cref{asn-vcx} we assume the optimization problem is well-parametrized in covariates: the optimization solution is well-described as a function of $x(W)$. We also require that there is a sensible way of sampling datapoints and projecting the feasible set $\mathcal{X}$ onto each subsampled index set. E.g. when subsampling in a matching example, after subsampling nodes the new feasible set in each index set $\mathcal{X}_1,\mathcal{X}_2$ preserves all edges between nodes in the original feasible set $\mathcal{X}$. 

Let $\Gamma_t(O_i; e, \mu)$ denote the score associated with observation $O_i = (W_i, T_i, c_i)$ under either IPW or AIPW, with input nuisance functions $e,\mu$. For example, as in \Cref{sec-estimation}, ${\Gamma^{\text{IPW}}_t(O; e, \mu)=\frac{\mathbb{I}[T=t] c}{e_t(w)}}$ and ${\Gamma^{\text{AIPW}}_t(O; e, \mu)=\frac{\mathbb{I}[T=t] (c-\mu_t(w)}{e_t(w)}} + \mu_t(w)$. 

As is standard in cross-fitting we use distinct main folds in order to estimate nuisances (indexed by parameters $\perturbparam$) for input into $\hat x$: that is, $$\hat x_1(\hat\perturbparam_2 ;W) \in \argmin_{x(W)\in\mathcal{X}_2} \frac{1}{\vert \mathcal{I}_{2} \vert} \sum_{i \in \mathcal{I}_{2}}  \sumt \pi_t(W_i) \Gamma_t(O_i; e^{-k(i)}, \mu^{-k(i)}) x_i,$$
and analogously for $\hat x_2$. 

Then evaluation estimates the value within each fold using the optimal solution from the other fold: 
$$\hat v_1 = \frac{1}{\vert \mathcal{I}_{2} \vert} \sum_{i \in \mathcal{I}_{2}} \sumt \pi_t(W_i) \Gamma(O_i; e^{-k(i)}, \mu^{-k(i)}) \hat{x}_1(\hat\perturbparam_1,W_i),$$
and we return the average over folds, $\frac{1}{2} (\hat v_1 + \hat v_2)$ or more generally the average of $\{ \hat v_{k}\}_{k\in K}.$ 

When $K > 2$, for each fold $k$, we will use two subfolds $I_{-k, e}$ and $I_{-k, \mu}$ to estimate $e, \mu$, and then obtain $\hat x_{-k}$ from $I_{-k}$. We evaluate the estimated objective with $\hat x_{-k}$, $\hat e_k$ and $\hat \mu_k$ and average over all folds. 

\begin{proposition}[Unbiased estimation by sample splitting.]\label{prop-sample-splitting}
$\frac{1}{K} \sum_i^K \hat v_k = \E[\oraclemuv] $
\end{proposition}
\begin{proof}
Immediate from standard analysis of AIPW and sample-splitting of \cite{ito2018unbiased}. 
\end{proof}

\subsection{Comparison of estimation properties in the two regimes (\Cref{table:result})}\label{sec:appendix-ipw}

\paragraph{Out-of-sample, fixed-dimension.}
IPW/AIPW-type estimators cannot be applied in the out-of-sample regime of \Cref{asn-asymptotic regime}, by definition of the regime. However, we may obtain out-of-sample risk bounds on the decision regret in this regime, simply by virtue of out-of-sample generalization risk bounds on the generated regressors. For example, we effectively assume near-parametric regimes for the propensity score so that the conditions of Theorem 1 of \citet{bertail2021learning}, providing a generalization risk bound for two-stage reweighted empirical risk minimization with estimated weights (as in our \Cref{sec-estimation}), are met. Under assumption of uniformly bounded decision variables, applying the Cauchy-Schwarz inequality directly implies that statistical estimation consistency of our estimation approaches imply decision regret consistency, so that the estimation bias vanishes at a $O_p(n^{-\frac 12})$ rate. (However the statistical rate of optimization bias adjustment remains unclear). 

\paragraph{In-sample, growing-dimension, growing-n.}
An analogous extension to sample splitting as in \citet{ito2018unbiased} is possible in highly structured situations satisfying \Cref{asn-vcx}. 
For a \textit{fixed} optimization solution $x$, uniform generalization over $\pi\in\Pi$ is a consequence of uniform generalization with a stochastic (bounded) envelope function. 
However, in this regime, uniform generalization over both $\pi\in \Pi$ and $\z\in\mathcal X$ is difficult because in the regime of \Cref{asn-insample-regime}, the dimension of the optimization grows as $n\to\infty$. Typical approaches to uniform convergence would require $x^*(\pi,W)$ (the optimal optimization solution at a fixed $\pi$) to converge uniformly over the space of policies and $W$. 

\paragraph{Different estimation interpretations of $\GRDR$ in the two regimes.}
Note that benefits of $\GRDR$ in terms of doubly-robust estimation of the ATE (mixed-bias, rate double-robustness) are only relevant in the \textit{in-sample regime} of \cref{asn-insample-regime}. Recent work does show this specification obtains empirical benefits for confounded outcome estimation, in appeal to the sufficient balancing properties of the propensity score, that may also apply to the regime of \cref{asn-asymptotic regime}.

\clearpage

\section{Beyond Linearity: Decision-Dependent Classifier Risk}\label{sec:decdependentclassifier}

The downstream optimization can also in turn be a prediction risk problem: treatments shift distributions upon which predictive risk models are trained \citep{paxton2013developing}. For example, the medical system simultaneously treats individuals but is also interested in large-scale predictive models from passively collected electronic health records, trained upon the realizations of health outcomes of the entire population, and so may generate distribution shifts in these predictive risk models \citep{agniel2018biases,finlayson2021clinician}. Therefore, the post-treatment predictive risk model introduces a downstream causal-policy dependent optimization response.

In the previous sections, we focused on linear optimization because plug-in estimation is consistent when the random variable enters linearly into the optimization problem. The challenge with nonlinearity is that such plug-in-approaches are no longer consistent and can introduce policy-dependent nuisance estimation functions.

Nonetheless, special structure of the problem can admit alternative estimation strategies.

\subsection{Problem setup} 

\begin{example}[{Decision-dependent classifier drift.}]\label{ex:decision-dependent-classifier}
We shift to notation more typical in statistics/machine learning to emphasize the setting. We model decision-dependent shift of predictive risk models in a repeated measurement setting.\footnote{That is, observing covariates and outcomes from the same unit, measured at different time periods.} Our observation trajectories\footnote{We are therefore modeling ``feedback loops" between outcomes $Y_0$ and the treatments administered to manage them via temporally distinct repeated measurements.} each comprise of $(L_0, Y_0, T_0, L_1(T_0), Y_1(T_0))$: baseline covariates $(L_0, Y_0)$, time-$0$ treatment $T_0\in\{0,1\}$, and post-treatment covariates and outcome $(L_1(T_0), Y_1(T_0))$. 

For example, $L$ could measure patient state and $Y$ a cardiac event within a given time period. Upon observation of $(L_0, Y_0)$ a patient is treated with $T_0$; for example with more aggressive or wait-and-see treatment depending on $Y_0(T_0)$. While optimal treatment regimes focus on averages of individual-level outcomes $Y_1(T_0)$; in our policy-dependent response setting we model the problem of, for example, continuously monitoring ``feedback loops" that may surface in predictive risk models that may generally be trained using large electronic health record databases. Said differently, we could have modeled this abstractly as a policy evaluation with the augmented set of ``covariates", jointly $(L_0, Y_0)$, and ``outcomes'' $(L_1(T_0), Y_1(T_0))$. However, we focus on the ultimate downstream predictive risk model which depends (nonlinearly) on all the outcomes of the population's units. The policy evaluation problem could evaluate the predictive loss of the downstream predictive model $f(L_1(\pi), \beta)$, where there is downstream optimization of the squared loss over $\beta$. 
\begin{equation}
    \min_\beta \E[(Y_1(\pi)- f(L_1(\pi), \beta))^2  ].
\end{equation}

The causal graph of \Cref{fig-decdep-classifier} describes the two-stage observation of individuals\footnote{We do not consider for now the causal effects of the prediction model, although this could be a direction for future work.}, comprising observation trajectories $(L_0, Y_0, T_0, L_1(T_0), Y_1(T_0))$. 
We consider in the general case a two-stage setting with a treatment affecting covariates and outcomes; upon which a predictive risk model is trained. 
\end{example}

\begin{example}[Policy optimization for \Cref{ex:decision-dependent-classifier}, policy-dependent prediction.] 
Consider the case of two different treatments with similar (conditional) average treatment effects, but one induces higher variability in outcomes which increases the fundamental noise level in the regression: harming the population prediction model and incurring higher loss. Optimizing between these two treatments, scalarizing population outcomes by the global term would result in choosing the less variable treatment. 

Therefore, in this framework we may be interested in the following scalarized policy optimization problem:
\begin{align*}
    \min_{\pi}
    \min_\beta \lambda \E[Y_1(\pi)] + (1-\lambda)\E[(Y_1(\pi)- f(L_1(\pi), \beta))^2  ].
\end{align*}
\end{example}

\subsection{Estimation }

\begin{figure}[!t]\label{fig-decdep-classifier}
\centering
\begin{tikzpicture}[%
			>=latex',node distance=2cm, minimum height=0.75cm, minimum width=0.75cm,
			state/.style={draw, shape=circle, draw=black, fill=green!2, line width=0.5pt},
			yy/.style={draw, shape=circle, draw=black, fill=purple!2, line width=0.5pt},
			action/.style={draw, shape=rectangle, draw=gray, fill=gray!2, line width=0.5pt},
			scale = 1, transform shape
			]
			\node[state] (X0) at (2,0) {$L_0$};
			\node[yy] (Y0) at (0,2) {$Y_0$};
			\node[state,right of=X0] (X1) {$L_1$};
			\node[action,above of=X0] (T) {$T$};
			\node[yy,above of=X1] (Y1) {$Y_1$};
			\draw[blue, ->] (X0) -- (T);
			\draw[thick, ->] (X0) -- (Y0);
			\draw[blue, ->] (Y0) -- (T);
			\draw[thick, red, ->] (T) -- (X1);
			\draw[blue, ->] (X0) -- (X1);
			\draw[thick, red, ->] (T) -- (Y1);
			\draw[blue, ->] (X1) -- (Y1);
			\draw[blue,-] (Y0) -- (0, 3);
			\draw[blue,-] (0,3) -- (4,3);
			\draw[blue,->] (4,3) -- (Y1);
			\draw[blue, -] (Y0) -- (0, -1);
			\draw[blue, -] (0, -1) -- (4, -1);
			\draw[blue, ->] (4, -1) -- (X1);
		\end{tikzpicture}
		\caption{Causal diagram for decision-dependent classifier drift.}
\end{figure}
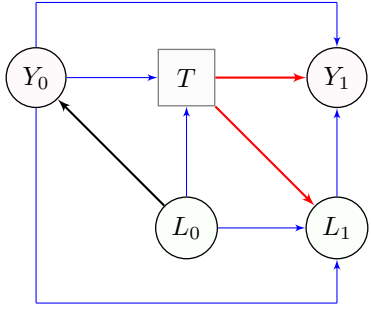


\paragraph{Plug-in estimators.}

Our goal is off-policy evaluation of predictive risk model (parameter) solving the downstream prediction risk optimization after treatment under treatment regime $\pi$: 

    {\centering
  $ \displaystyle
\begin{aligned}
\beta^* \in \arg\min_{\beta} \E[(Y_1(\pi)-f(L_1(\pi); \beta)^2]. 
\end{aligned}
$\par}

However for the special case that we consider of linear regression response (squared loss, linear parametrization), we may orthogonalize the \textit{first-order optimality} conditions of the policy-dependent optimization, e.g. recognizing that, $\beta^*$ solves the first order conditions 
\begin{equation} 
\E[ L^\top_1 (Y_1 - L_1\beta)] = 0. \label{eqn-est-eqn}
\end{equation} 
Hence for the case of least squares and linear regression, we may focus on estimation refinements for $\beta$: observe that the estimation requires estimation of certain transformations of $(X_1,Y_1)$ unit's downstream outcome, $x_1y_1$ and $x_1^\top x_1$, and we may estimate the following matrix-regression or vector-valued regression nuisance estimates: 
\begin{align*}
\E[ L_1Y_1 \mid T=t, X_0, Y_0], \qquad \E[ L_1 L_1^\top \mid T=t, L_0, Y_0] 
\end{align*}
Hence standard AIPW-type approaches can be applied with the above nuisances and the censored observations $l_1y_1(t), l_0y_0(t)$. 

This suggests that when $\beta$ is our parameter of interest (or functions thereof), we can leverage double robustness. And, if we have a small space of policies we can optimize by enumeration. However the same challenges regarding nonlinearity remain if we want to estimate the final squared loss of $\theta$. 

\begin{remark}[Restriction to linear models and the challenge for generalization to nonlinear models]
Note that the challenge with generalizing to non-least-squares losses or nonlinear predictors is that due to nonlinearity, doubly-robust estimation of outcome $X_1$ need not provide the same benefits of bias reduction. Although an alternative approach is to instead estimate the \textit{squared loss} as the composite outcome ${\E[(Y_1(t)-\theta^\top L_1(t))^2\mid t, L_0, Y_0]}$ because of our \textit{policy-dependent response} optimizing over $\theta$, we would have policy-dependent nuisance functions so this becomes intractable.
\end{remark}
\clearpage
\section{Additional Experiment Details and Results}\label{app:exp}

In this section, we provide more details on the experimental setup as well as further results. 

\subsection{Causal effect estimation setup}

Fr the causal effect estimation we generated the training dataset ${\cD_1 = \{(W, T, c)\}}$ with covariate $W \sim \cN(0,1)$, confounded treatment $T$, and outcome $c$. Treatment is drawn with probability ${\pi_t^b(W) = \frac{\mathrm{1}}{\mathrm{1} + e^{-\pith_1 W + \pith_2}} }$, $\pith_1 = \pith_2 = 0.5$. The true outcome model is given by a degree-2 polynomial:
\[poly_{\theta}(t, w) = (1, w, t, w^2, wt, t^2)\cdot([5,1,-1,2,2,-1])^\top.\]
We generate the outcome samples as ${c_t(w) = poly_{\theta}(t, w) + \epsilon}$, where $\epsilon \sim \cN(0,1)$. All random samples are generated using \texttt{numpy.random} package. In the mis-specified setting that induces confounding, the outcome model is a vanilla linear regression over $W$ without the polynomial expansion.

In \Cref{fig:in-sample-CATE-a} and \Cref{fig:in-sample-CATE-b} in the main text, we illustrate the (covariate-conditional) estimation over the covariates' landscape for the direct method, the weighted direct method ($\WDM$), and the doubly robust method ($\GRDR$) when there is a model mis-specification. We provide the estimation results without model mis-specification in \Cref{app:fig:in-sample-CATE-well}.

\begin{figure}[!ht]
    \centering
    \includegraphics[height=0.35\textwidth]{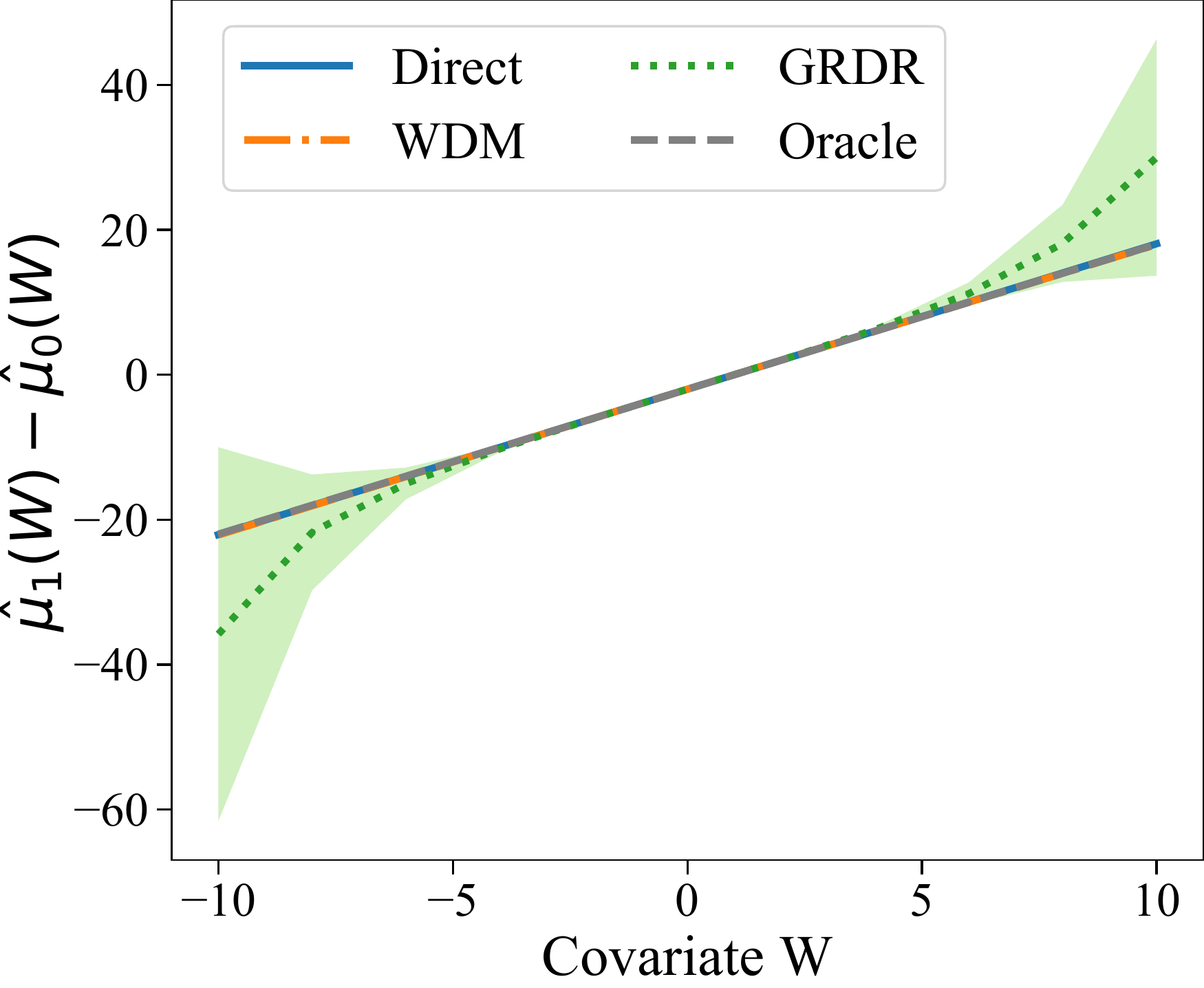}
     \caption{\emph{(In-sample estimation of $\hat\mu_1(W) - \hat\mu_0(W)$, no model mis-specification)}.  Comparison of direct / weighted direct ($\WDM$) / doubly robust method ($\GRDR$) to the oracle estimator for estimation of conditional ATE over different covariate values. Results are averaged over ten random training datasets; shading area indicates the standard error. 
    }
    \label{app:fig:in-sample-CATE-well}
\end{figure}

When there is no model mis-specification that that induces confounding, we observe that both the three estimation methods perform well against the oracle estimation.

\subsection{Policy evaluation}

For policy evaluation, we compare the perturbation method (Algorithm~\ref{alg:perturbation-method}) when being applied with three different estimators (direct, $\WDM$, and $\GRDR$). For consistency, throughout the evaluations, we follow the same true outcome model and covariate distribution as in the previous subsection for causal effect estimation. In both the well-specified model setting and the mis-specified model setting, the mean-squared-error (MSE) of the estimated policy value with the three estimators is computed with regard to the ground truth outcome model (aka oracle).

When evaluating Algorithm~\ref{alg:perturbation-method}, we generated $S=20$ bootstrap replicates. The downstream matching problem is evaluated with $m=500$ left-hand-side nodes, and $m'=300$ right-hand-side nodes. The min-cost matching requires each node to be matched to no more than one node on the other side, and was computed by the \texttt{linear\_sum\_assignment} function of the \texttt{scipy.optimize} package in Python 3. We evaluated a fixed logistic policy $\pi_t(W) = sigmoid(\phi \cdot W + b)$ with $\phi=1, b = 0.5$.

\begin{figure}[!ht]
    \centering
    \includegraphics[height=0.35\textwidth]{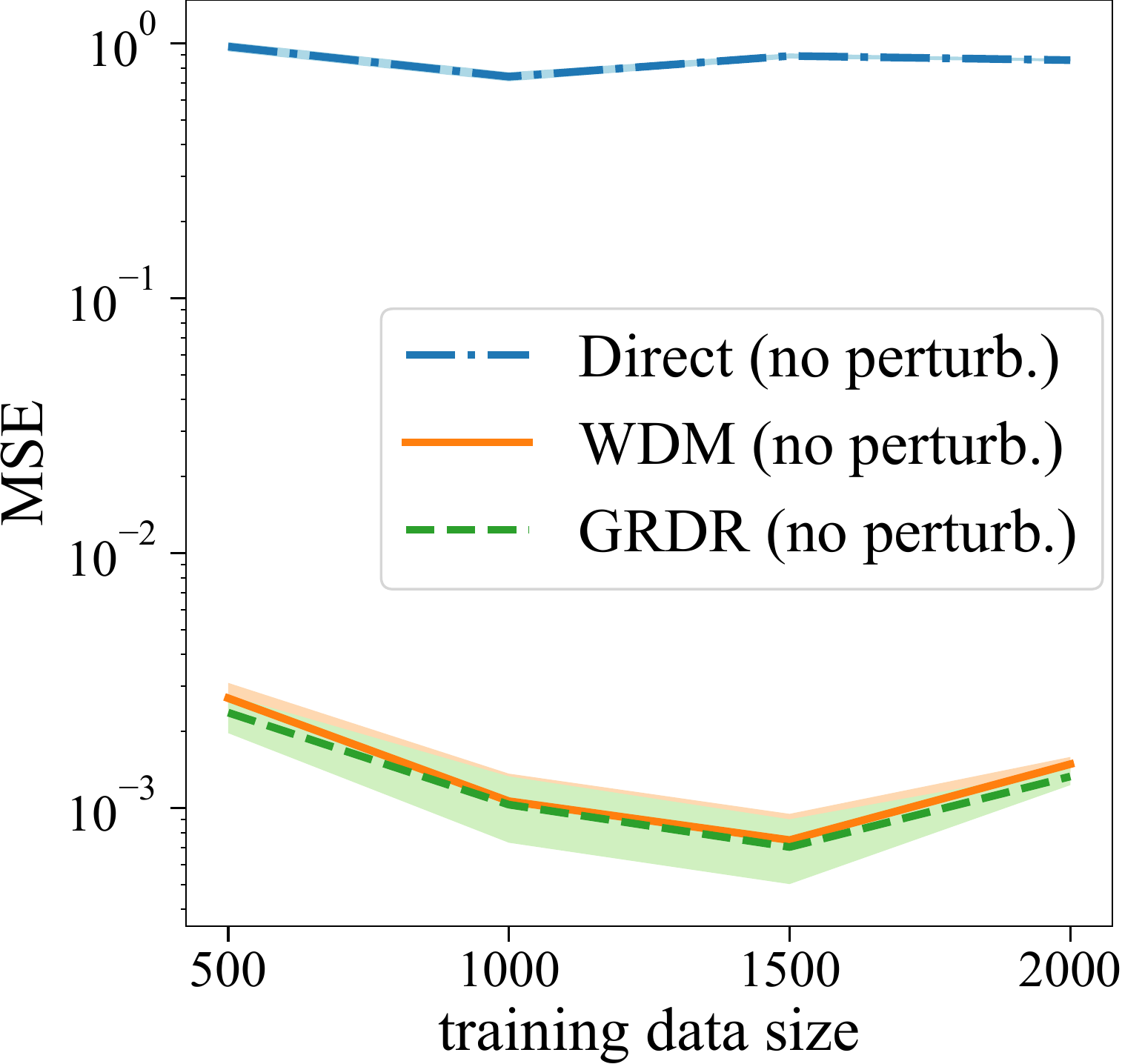}
     \caption{\emph{(In-sample estimation of $\hat\mu_1(W) - \hat\mu_0(W)$ with model mis-specification, no perturbation applied)}.  Comparison of direct / weighted direct ($\WDM$) / doubly robust method ($\GRDR$) over increasing size of training data. Results are averaged over ten random training datasets; shading area indicates the standard error. 
    }
    \label{app:fig:pol-eval-mis-np}
\end{figure}

Figure~\ref{fig:main-pol-eval} shows that when there is mis-specification, even a large training dataset cannot bring bias correction for the direct method, where both $\WDM$ and $\GRDR$ enjoy smaller and decreasing MSE. As an ablation study, we also compare to the corresponding performance in the mis-specified setting when we do not perform the perturbations (i.e. no bootstrapping in Alg.~\ref{alg:perturbation-method}). In detail, we directly return $\hat v^{(0)}$ without doing the later bootstrap procedure. Figure~\ref{app:fig:pol-eval-mis-np} indicates that the perturbation method is helpful for MSE reduction for both $\WDM$ and $\GRDR$. 

We further conduct evaluations with different bootstrap replicates' sizes (controlled by variable $h$ in \Cref{alg:perturbation-method}). We include these results in Table~\ref{app:table:pol-eval-h}. Results in Table~\ref{app:table:pol-eval-h} show that $\WDM$ and $\GRDR$ remain more superior and that is robust with different replicate sizes. For the evaluations over different $h$ values, we used training data with $3000$ samples. In each iteration, the number of bootstrap replicates is 20.

\begin{table*}[]
 \caption{\emph{(Perturbation method, varying replicate size.)} Performance for different estimator/model combinations. Mean-squared-errors (MSE) are computed with regard to the oracle outcome model. \vspace{0.1cm}}
 \centering
 \begin{adjustbox}{max width=0.9\textwidth}
\begin{tabular}{clllll}
\toprule
\multicolumn{1}{l}{\textbf{}}           & Estimation       & $h=1$ & $h=2$ & $h=3$ & $h=4$ \\
          \midrule
\multirow{3}{*}{Mis-specified model} & Direct &\color{negvals}{0.59$\pm$0.05}    & \color{negvals}{0.70$\pm$0.05}   &  \color{negvals}{0.76$\pm$0.06}   & \color{negvals}{0.70$\pm$0.06}   \\
                                        & $\WDM$ &0.00031$\pm$0.0001   &0.00042$\pm$0.0002     & 0.00041$\pm$0.0002    & 0.00048 $\pm$0.0003  \\
                                        & $\GRDR$ &0.00040$\pm$0.0002     &0.00046$\pm$0.0002     & 0.00031$\pm$0.0001    & 0.00035$\pm$0.0001    \\ 
          \midrule
\multirow{3}{*}{Well-specified model}         & Direct & \color{posvals}{0.00079$\pm$0.0004}    &  \color{posvals}{0.00067$\pm$0.0004}   &  \color{posvals}{0.00062$\pm$0.0003}   &  \color{posvals}{0.00024$\pm$0.0002}   \\
                                        & $\WDM$    & 0.00076$\pm$0.0004    &  0.00067$\pm$0.0003   &  0.00080$\pm$0.0002   & 0.00031$\pm$0.0001    \\
                                        & $\GRDR$   & 0.00082$\pm$0.0002    &  0.00067$\pm$0.0002   &0.00080$\pm$0.0002     & 0.00031$\pm$0.0001    \\ \bottomrule
\end{tabular} \label{app:table:pol-eval-h}
\end{adjustbox}
\end{table*}

\subsection{Policy optimization}

For policy optimization, we implemented the subgradient method as in \Cref{alg:subgradient}, and obtained causal effect estimators from \Cref{alg:perturbation-method}. 

In detail, for a given training dataset, we first obtained $S+1$ outcome estimators (i.e. \{$\hat\mu_t^\genest(w_i; \hat\perturbparam_\genest), \hat\mu_t^\genest(w_i; \hat\perturbparam_\genest)^{(j)}, j =1\cdots S\}$ in \Cref{alg:perturbation-method}) via bootstrap. Then, at each iteration of running subgradient descent, we evaluate the current policy using the $S+1$  outcome estimators respectively, and obtain $S+1$ subgradients of it. We then aggregate these subgradients by the bootstrap aggregation (as in Step 6, \Cref{alg:perturbation-method}).

We evaluate subgradients of the inner optimization solution in \cref{alg:subgradient} (step 4) by evaluating the gradient of the objective with respect to $\pith$, fixing the inner optimization variable $\z^*$. The fact that $\nabla_\pith$ is a subgradient is a consequence of Danskin's theorem \citep{danskin1966theory}. The inner minimization (the matching problem) is again solved by the \texttt{linear\_sum\_assignment} function of the \texttt{scipy.optimize} package in Python 3. 

\begin{figure}[!ht]
    \centering
     \subfloat[Policy optimization with well-specified models]{\includegraphics[height=0.3\textwidth]{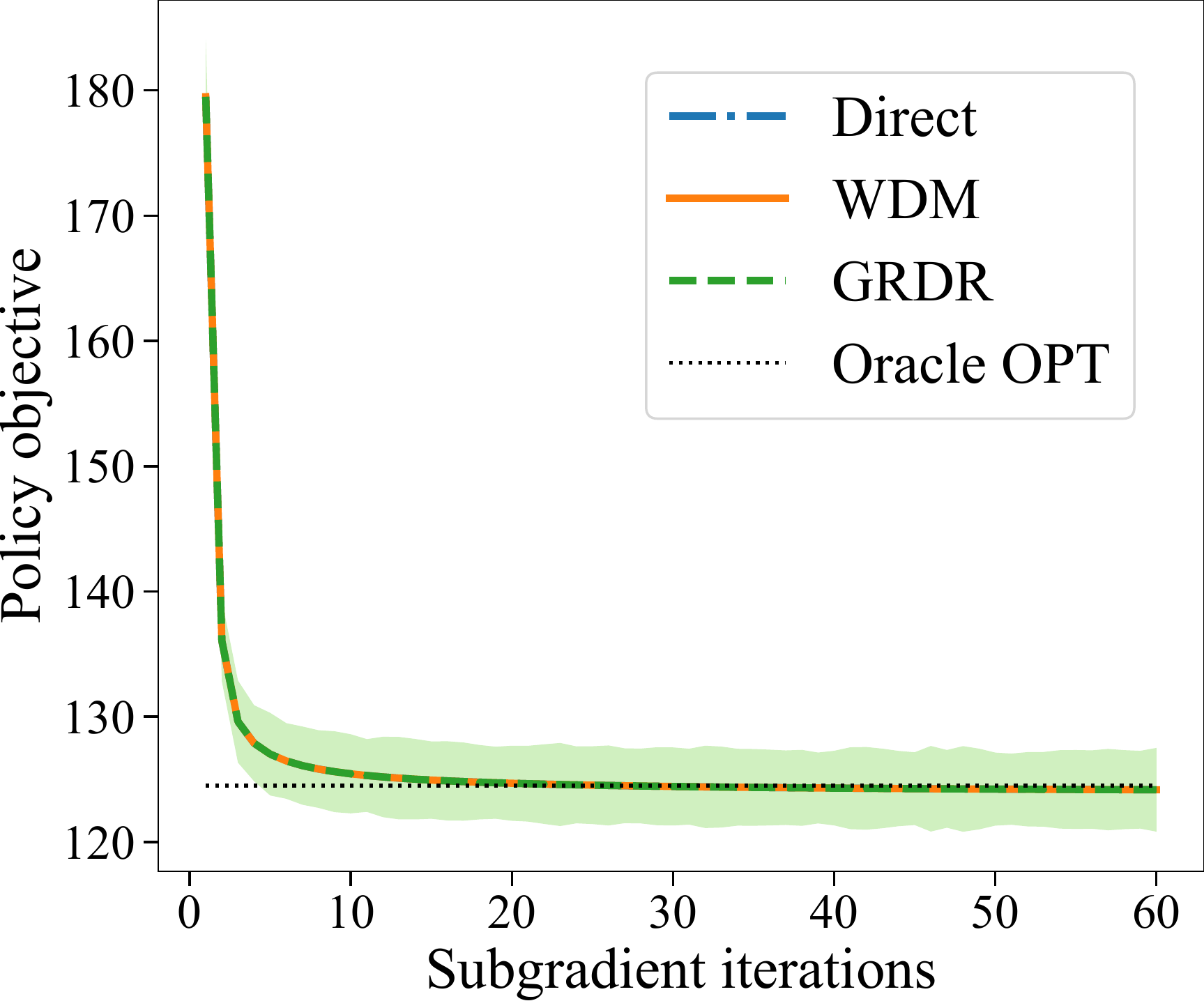}\label{fig:pol-opt-well-same-init}} \quad \quad
    \subfloat[{\centering Policy optimization with model mis-specification}]{\includegraphics[height=0.3\textwidth]{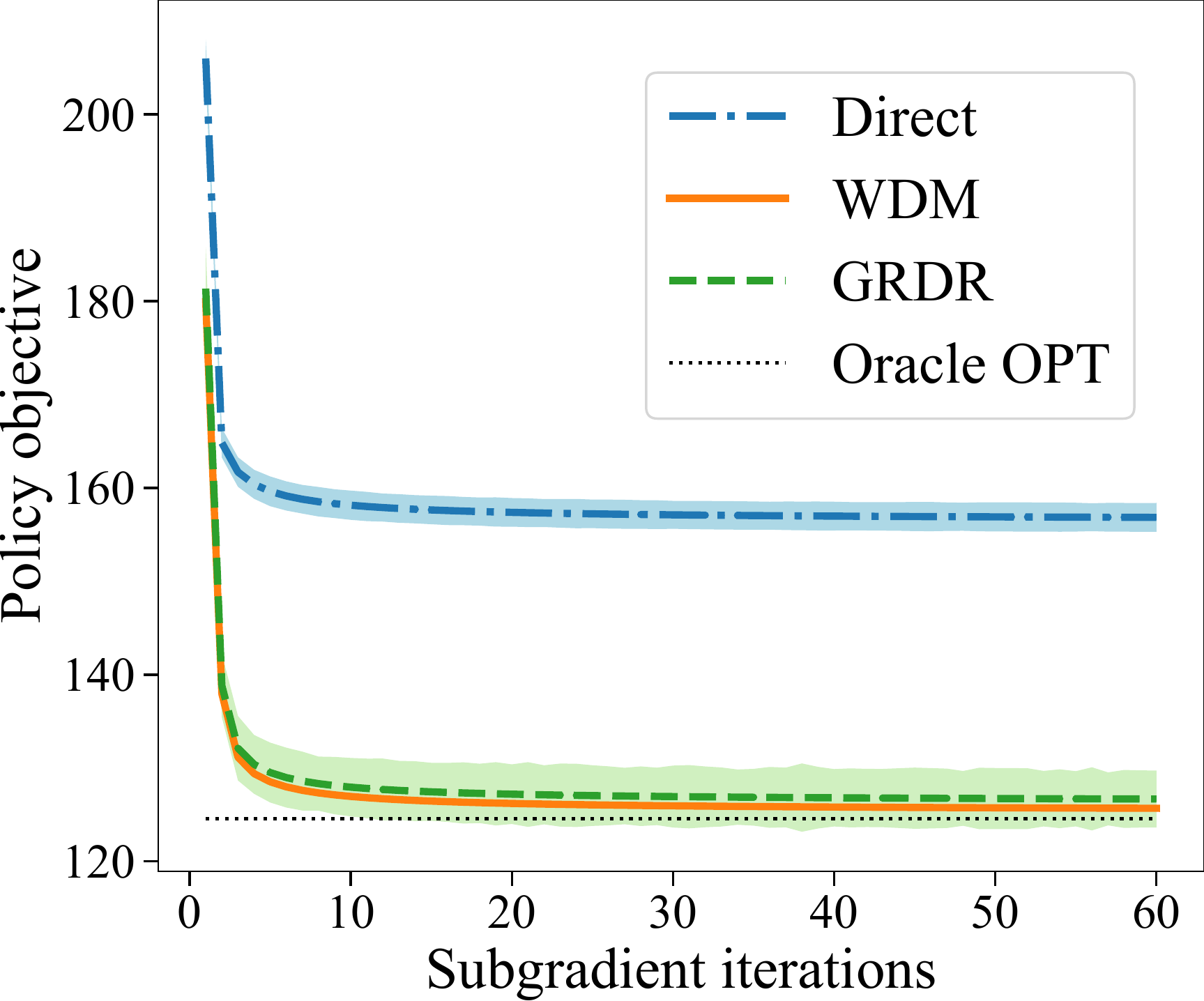}\label{fig:pol-opt-mis-same-init}}
     \caption{\emph{(Policy optimization (fixed test set))}.  Results of subgradient policy optimization with direct / weighted direct ($\WDM$) / doubly robust ($\GRDR$) estimation methods and a fixed test set. Averaged over ten random training datasets of size 1000.
    }
    \label{fig:pol-opt-fixed-test-same-init}
\end{figure}

To further study the impact of the random initial policies to begin with the subgradient descent algorithm, in Figure~\ref{fig:pol-opt-fixed-test-same-init} we obtained the corresponding results of Figure~\ref{fig:pol-opt-fixed-test}, but with a fixed initial policy. We observe that again $\WDM$ and $\GRDR$ quickly converges to the oracle estimation, while the large bias of the direct method leads to poor policy optimization. Moreover, in this relatively low-dimension example, although random initialization of the policy leads to a larger variance in earlier iterations, the policy value convergences to oracle policy objective quickly after a few iterations. 

For the evaluations of policy optimization, we used training datasets with size $1000$, and a downstream min-cost matching with $m=100, m'=60$. The learning rate was tuned over $[0.01, 0.1, 1]$. All of our evaluations were run on a 2.3 GHz 8-Core Intel Core i9 CPU. All the differentiation operations were handled by the automatic differentiation library in \texttt{JAX}. 

\subsection{Additinal comparisons and evaluations}

\paragraph{Additional evaluations with more complex non-linear outcome model.}

We conduct further robustness checks with nonlinear data-generating processes: 
exponential and quadratic. The outcome is $c_t(w) = a_1 + a_2 \exp(b_1 + b_2 w) + c_1 t +c_2 tw^2 + \epsilon$, where $\epsilon \sim \mathcal{N}(0,1)$ is an external noise, and $[a_1, a_2, b_1, b_2, c_1, c_2] = [ 5,  0.05,  0.5, -2, -2, -1 ]$. 
We fit this function with nonlinear least squares (\textit{scipy.optimize.curve\_fit}). Indeed, Figure~\ref{fig:exp-function} shows that the direct method is sensitive to the model mis-specification bias without using a near-specified nonlinear curve fit. However the weighted direct method (WDM) and the doubly robust estimator (GRDR) remain robust; even if starting with misspecified parametric models.

\begin{figure}[H]
    \centering
    \subfloat[\centering]{{\includegraphics[width=0.22\textwidth]{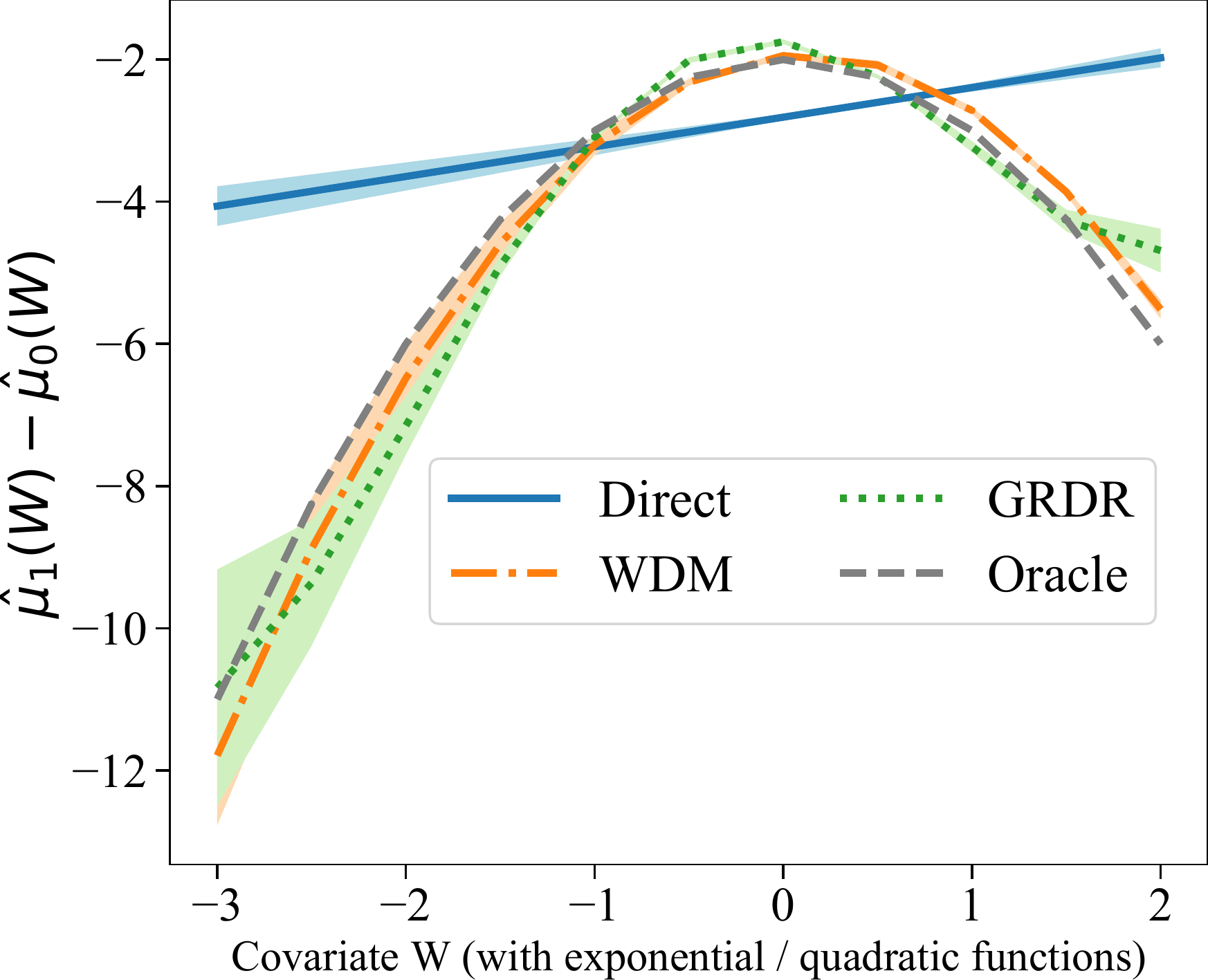}}}%
    \qquad
    \subfloat[\centering]{{\includegraphics[width=0.22\textwidth]{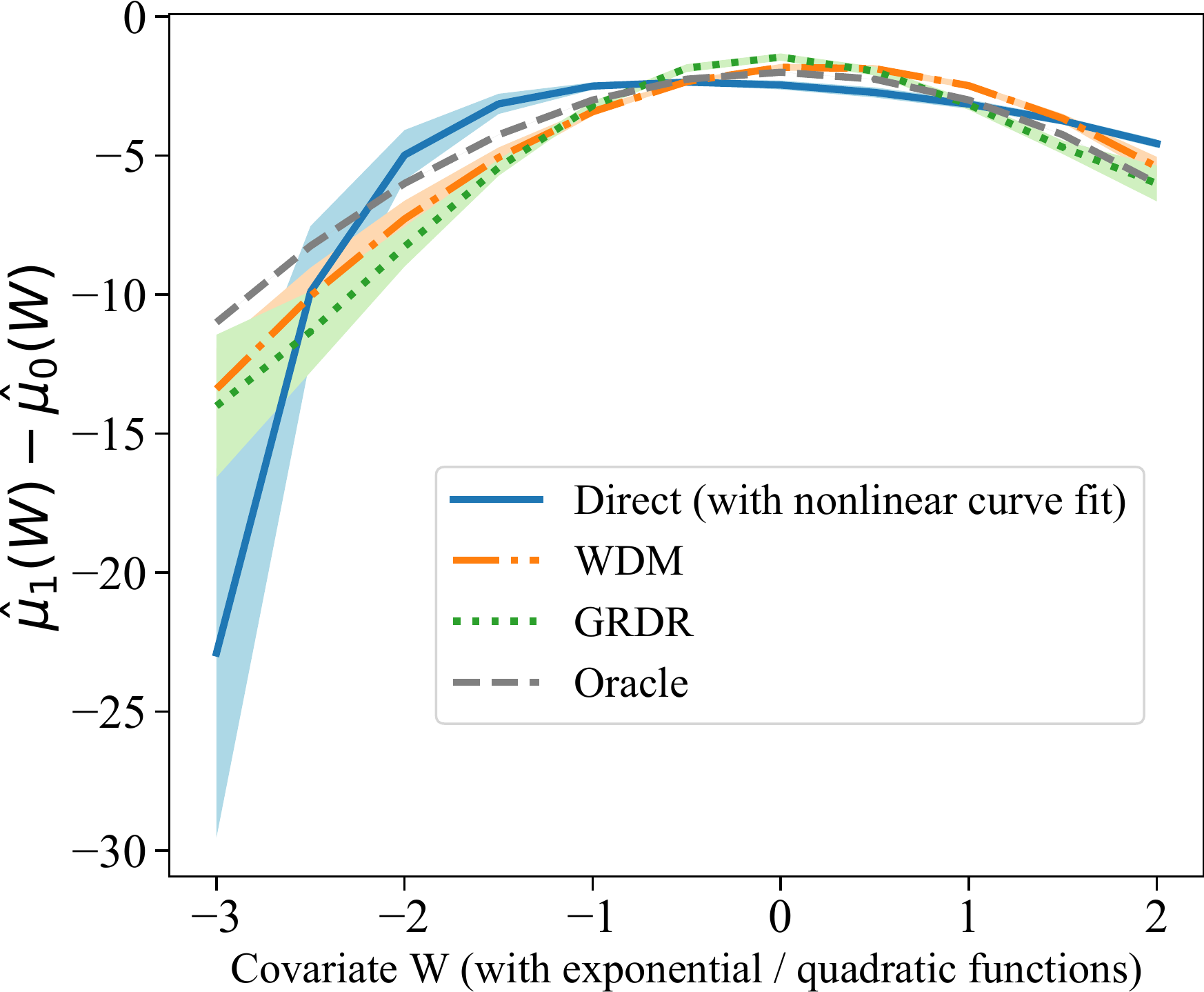} }}%
        \subfloat[\centering]{{\includegraphics[width=0.22\textwidth]{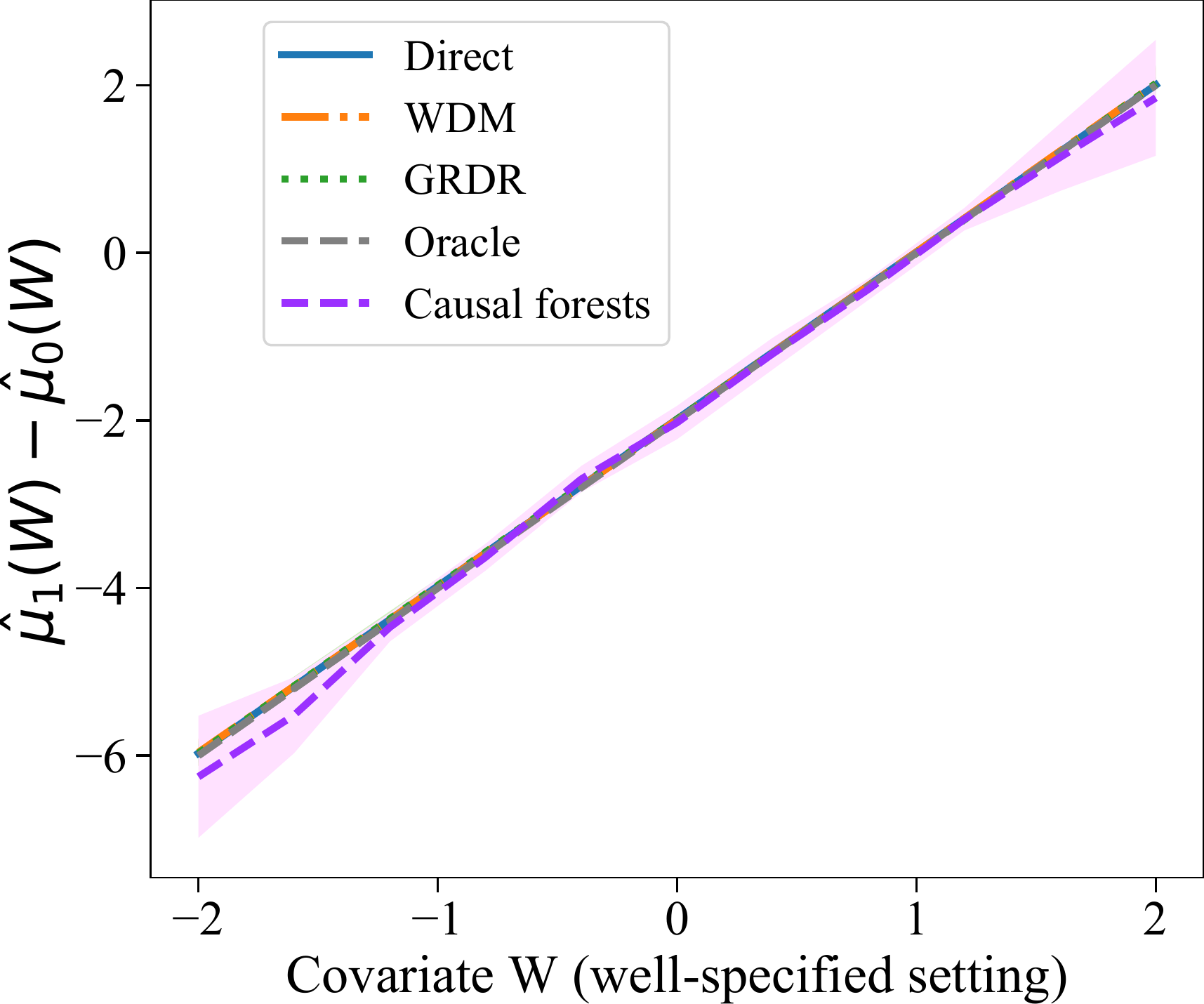}}}%
    \qquad
    \subfloat[\centering]{{\includegraphics[width=0.22\textwidth]{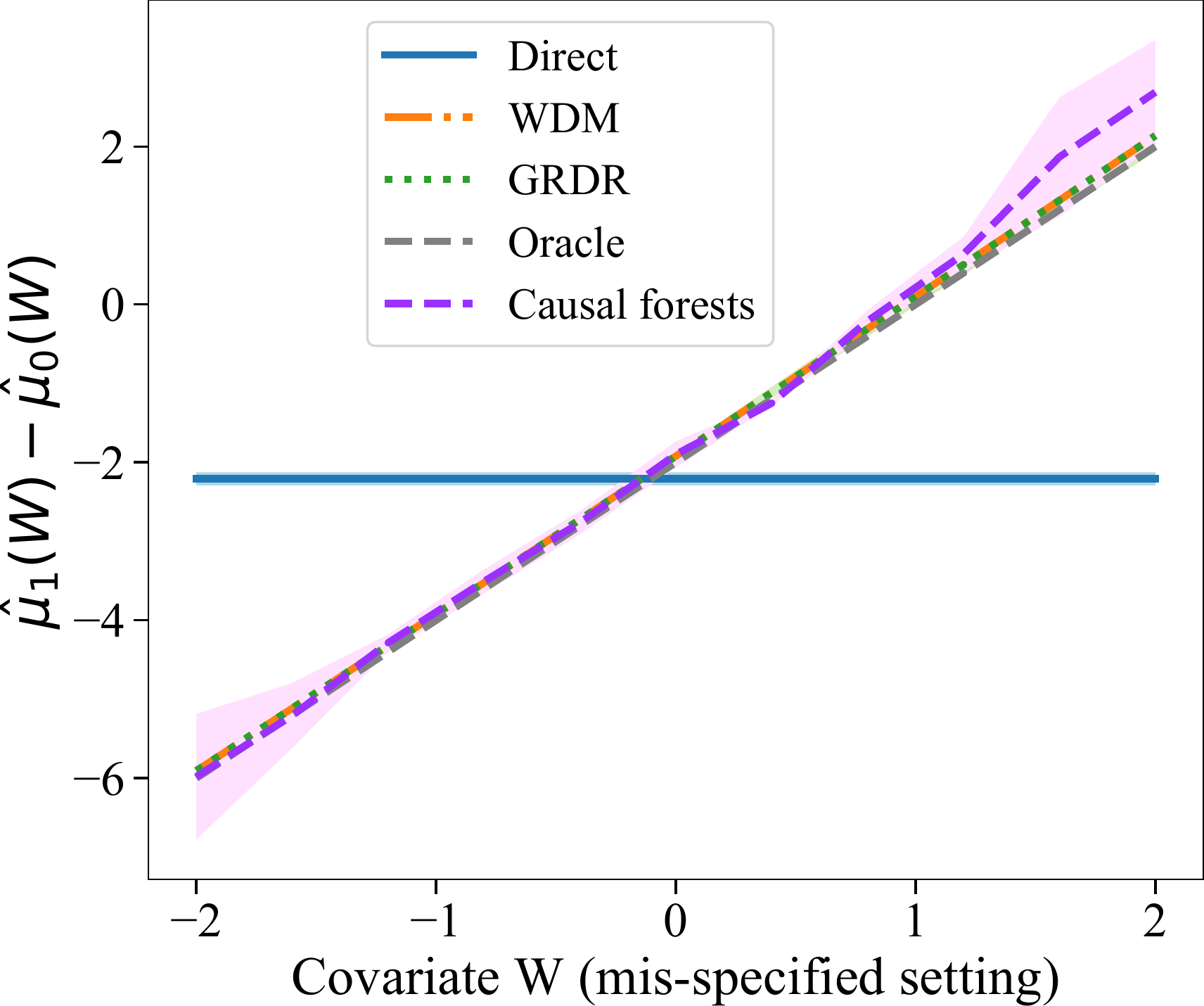} }}%
    \caption{(In-sample estimation of $\hat\mu_1(W) - \hat\mu_0(W)$ for exponential function, (a) without / (b) with curve fit. Comparisons of the CATE estimates with nonparametric estimators. (c) without / (d) with  model mis-specification.}%
    \label{fig:exp-function}%
\end{figure}

\paragraph{Additional comparison to non-parametric estimators.}

We further compare our estimators against existing CATE estimators (although in general, these estimators tuned to estimate contrasts may improve upon differencing outcome models). 
We compared the DM/WDM/GRDR to the Causal Random Forests estimator proposed in (Wager and Athey, 2018) \footnote{Based on implementation by Battocchi et al, \textit{EconML: A Python Package for ML-Based Heterogeneous Treatment Effects Estimation.}}, following the setup in Section 5.1. Moreover, we also compared how the CATE estimators affect the policy evaluation task with the perturbation method (section 5.2). In Figure\ref{fig:exp-function}(c,d) comparing CATE estimates, the non-parametric random forest estimator is indeed unbiased, while the naive direct method has a large bias under model mis-specification.

\begin{figure}[ht]
    \centering
    {{\includegraphics[width=0.35\textwidth]{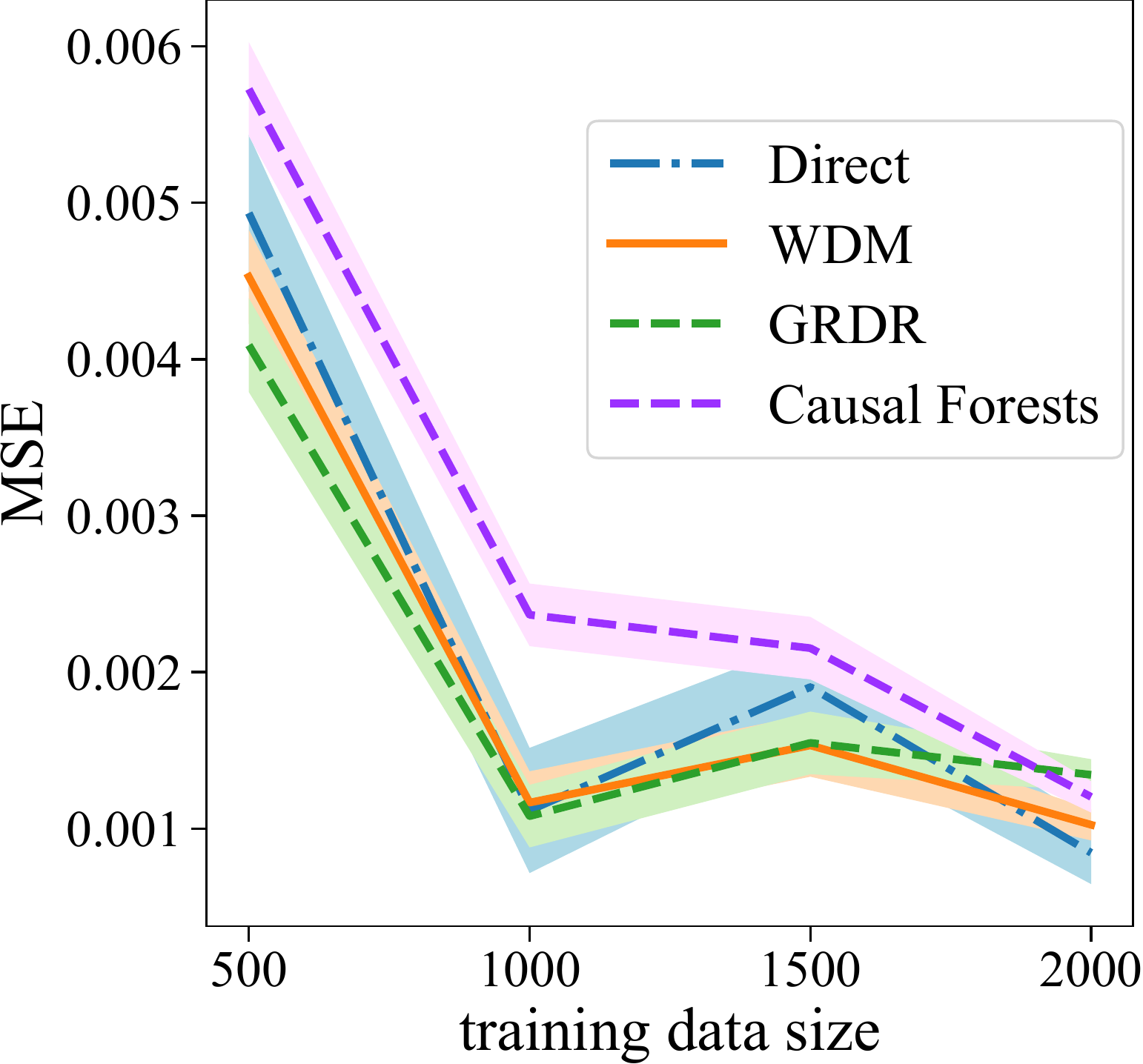}}}%
    \qquad
    {{\includegraphics[width=0.35\textwidth]{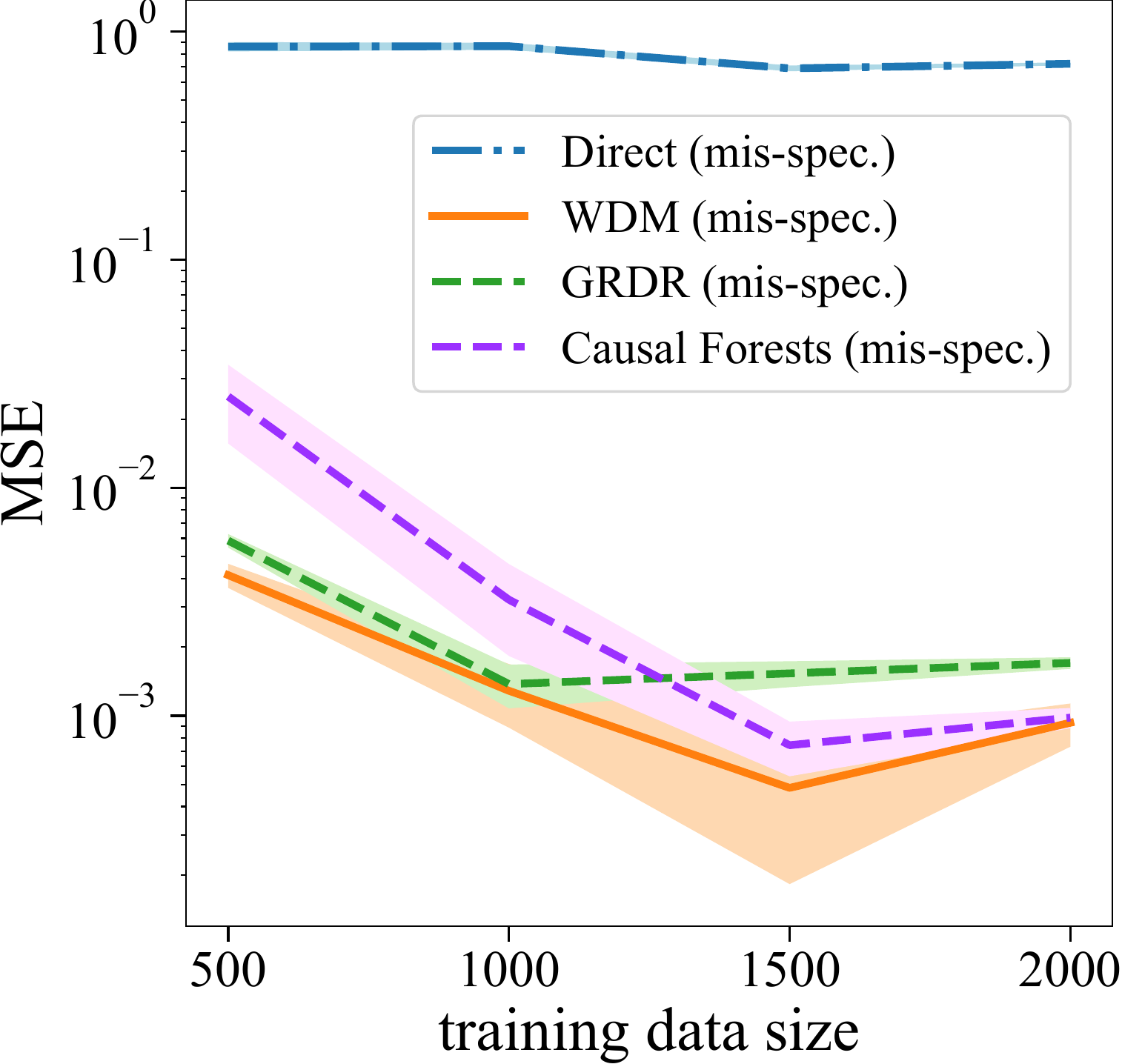} }}%
    \caption{\textit{(Policy evaluation via perturbation method (Algorithm 1)}.
Comparison of direct / WDM / GRDR / Causal Forests estimators over increasing size of training data.}%
    \label{fig:CATE-perturb}%
\end{figure}

\end{document}